%% file: main.tex
\documentclass[11pt]{article}
\usepackage{amssymb}
\usepackage{amsmath,amsfonts}
\usepackage{latexsym}
\usepackage{mathrsfs}
\usepackage{verbatim}
\usepackage{color}
\usepackage{enumerate}
\usepackage{amsthm}
\usepackage{pifont}
\usepackage{array}
\usepackage{graphicx}
\usepackage{epstopdf}
\usepackage{subfig}
\usepackage[toc,page]{appendix}
\usepackage[normalem]{ulem}
\usepackage[numbers,sort]{natbib}

\usepackage[ruled,vlined]{algorithm2e}
\makeatletter
\newcommand{\nosemic}{\renewcommand{\@endalgocfline}{\relax}}
\newcommand{\dosemic}{\renewcommand{\@endalgocfline}{\algocf@endline}}
\makeatother

\DeclareMathOperator*{\argmin}{arg\,min}
\DeclareMathOperator*{\argmax}{arg\,max}

\DeclareMathOperator{\btheta}{\boldsymbol{\theta}}
\DeclareMathOperator{\bmu}{\boldsymbol{\mu}}
\DeclareMathOperator{\balpha}{\boldsymbol{\alpha}}
\DeclareMathOperator{\bSigma}{\boldsymbol{\Sigma}}

\DeclareMathOperator{\bx}{\mathbf{x}}
\DeclareMathOperator{\bz}{\boldsymbol{\theta}}
\DeclareMathOperator{\bxi}{\boldsymbol{\xi}}

\DeclareMathOperator{\cala}{\mathcal{A}}

\DeclareMathOperator{\caln}{\mathcal{N}}
\DeclareMathOperator{\calc}{\mathcal{C}}
\DeclareMathOperator{\call}{\mathcal{L}}

\DeclareMathOperator{\bI}{\mathbf{I}}
\DeclareMathOperator{\bv}{\mathbf{v}}

\DeclareMathOperator{\equals}{\enskip = \enskip}

\DeclareMathOperator{\enplus}{\enskip+\enskip}

\usepackage{enumitem}

\newtheorem{theorem}{Theorem}
\newtheorem{corollary}{Corollary}
\newtheorem{lemma}{Lemma}
\newtheorem{remark}{Remark}
\newtheorem{proposition}[theorem]{Proposition}
\newtheorem{definition}{Definition}

\title{Asymptotic Analysis of Objectives based on Fisher Information in Active Learning}

\begin{document}
\author{Jamshid Sourati\thanks{Department of Electrical and Computer Engineering, Northeastern University, Boston MA.  E--mail: {\tt sourati@ece.neu.edu}} \and 
Murat Akcakaya\thanks{Department of Electrical and Computer Engineering, University of Pittsburgh.  E--mail: {\tt akcakaya@pitt.edu}} \and Todd K. Leen\thanks{Georgetown University.  E--mail: {\tt todd.leen@georgetown.edu}} \and Deniz Erdogmus\thanks{Department of Electrical and Computer Engineering, Northeastern University, Boston MA..  E--mail: {\tt erdogmus@ece.neu.edu}} \and Jennifer G. Dy\thanks{Department of Electrical and Computer Engineering, Northeastern University, Boston MA..  E--mail: {\tt jdy@ece.neu.edu}}}
\date{\today}
\maketitle
\thispagestyle{empty}

\begin{abstract}

Obtaining labels can be costly and time-consuming.  Active learning allows a learning algorithm to intelligently query samples to be labeled for efficient learning.  Fisher information ratio (FIR) has been used as an objective for selecting queries in active learning.  However, little is known about the theory behind the use of FIR for active learning.  There is a gap between the underlying theory and the motivation of its usage in practice.  In this paper, we attempt to fill this gap and provide a rigorous framework for analyzing existing FIR-based active learning methods.  In particular, we show that FIR can be asymptotically viewed as an upper bound of the expected variance of the log-likelihood ratio. Additionally, our analysis suggests a unifying framework that not only enables us to make theoretical comparisons among the existing querying methods based on FIR, but also allows us to give insight into the development of new active learning approaches based on this objective.


\end{abstract}

\section{Introduction}
\label{sec:intro}
\input{Intro}

Before going through the main discussion in section~\ref{sec:theoretical_analysis}, we formalize our classification model assumptions, set the notations and review the basics and some of the key properties of our inference method, maximum likelihood estimation, in sections~\ref{sec:framework_assumptions} and~\ref{sec:background}. The statistical background required to follow the remaining sections is given in Appendix~\ref{sec:statistical_background}.

\input{ClassificationModel}

\section{Fisher Information Ratio as an Upper Bound}
\label{sec:theoretical_analysis}
\input{TheoreticalAnalysis}

\section{Fisher Information Ratio in Practice}
\label{subsec:inpractice}
\input{InPractice}

\section{Conclusion}
\label{sec:conclusion}
In this paper, we focused on active learning algorithms in classification problems whose objectives are based on Fisher information criterion. As the primary result, we showed the dependency of the variance of the asymptotic distribution of log-likelihood ratio on the Fisher information of the training distribution. Then, we used this dependency to derive our novel theoretical contribution by establishing the Fisher information ratio (FIR) as the upper bound of such asymptotic variance. We discussed that several layers of approximations can be employed in practice to simplify FIR; simplifications, that can usually be avoided in pool-based active learning. Finally Monte-Carlo simulations and greedy algorithms can be used to evaluate and optimize the (simplified) FIR objective, respectively. Using this framework, we can distinguish the main differences between some of the FIR-based querying methods in the classification context. Such comparative analysis, not only shed light on the assumptions and simplifications of the existing algorithms, it can also be helpful for finding suitable directions in developing novel active learning algorithms based on the Fisher information criterion.

\begin{appendices}
\appendix
\input{Appendices.tex}
\end{appendices}

\vskip 0.2in
\bibliographystyle{abbrvnat}
\bibliography{myrefs}

\end{document}

%% file: Intro.tex
In supervised learning, a \emph{learner} is a model-algorithm pair that is optimized to (semi) automatically perform tasks, such as classification, or regression using information provided by an external source (oracle). In \emph{passive learning}, the learner has no control over the information given. In \emph{active learning}, the learner is permitted to \emph{query} certain types of information from the oracle~\citep{cohn1994improving}. Usually there is a cost associated with obtaining information from an oracle; therefore an active learner will need to maximize the information gained from queries within a fixed budget or minimize the cost of gaining a desired level of information. A majority of the existing algorithms restrict to the former problem, to get the most efficiently trained learner by querying a fixed amount of knowledge~\citep{settles2012active,fu2013survey}.

Active learning is the process of coupled querying/learning strategies. In such an algorithm, one needs to specify a query quality measure in terms of the learning method that uses the new information gained at each step of querying. For instance, information theoretic measures are commonly employed in classification problems to choose training samples whose class labels, considered as random variables, are most informative with respect to the labels of the remaining unlabeled samples. This family of measures is particularly helpful when probabilistic approaches are used for classification. Among these objectives, Fisher information criterion is very popular due to its relative ease of computation compared to other information theoretic objectives, desirable statistical properties and existence of effective optimization techniques. However, as we discuss in this manuscript, this objective is not well-studied in the classification context and there seems to be a gap between the underlying theory and the motivation of its usage in practice. \emph{This paper is an attempt to fill this gap and also provide a rigorous framework for analyzing the existing querying methods based on Fisher information.}


From the statistical point of view, we characterize the process of constructing a classifier in three steps as follows: (1) choosing the loss and risk functions, (2) building a decision rule that minimizes the risk, and (3) modeling the discriminant functions of the decision rule. For instance, choosing the simple 0/1 loss and its a posteriori expectation as the risk, incurs the Bayes rule as the optimal decision~\citep{duda1999pattern}, where the discriminant function is the posterior distribution of the class labels given the covariates. For this type of risk, discriminative models that directly parametrize the posteriors, such as logistic regression, are popularly used to learn the discriminant functions~\citep{bishop2006pattern}. In order to better categorize the existing techniques, we break an active learning algorithm into the following sub-problems:

\begin{enumerate}[label=(\roman*)]
\item\label{subproblem:query_selection} \emph{(Query Selection)} Sampling a set of covariates $\{\bx_1,...,\bx_n\}$ from the \emph{training} marginal\footnote{Throughout this paper, marginal distribution or simply distribution refers to the distribution of covariates, while joint distribution is used for pairs of the covariates and their class labels.}, whose labels $\{y_1,...,y_n\}$ are to be requested from an external source of knowledge (the \emph{oracle}). The queried covariates together with their labels form the \emph{training} data set.
\item\label{subproblem:passive_learning} \emph{(Inference)} Estimating parameters of the posterior model based on the training data set formed in the previous step.
\item\label{subproblem:prediction} \emph{(Prediction)} Making decisions regarding class labels of the \emph{test} covariates sampled from the \emph{test} marginal.
\end{enumerate}

These three steps can be carried out iteratively. Note that the query selection sub-problem is formulated in terms of the distribution from which the queries will be drawn. Ideally, queries (or the query distribution) are chosen such that they increase the expected quality of the classification performance measured by a particular objective function. This objective can be constructed from two different perspectives: based on the accuracy of the parameter inference or the accuracy of label prediction. In the rest of the manuscript, accordingly, we refer to the algorithms that use these two types of objectives as \emph{inference-based} or \emph{prediction-based} algorithms, respectively. 

Most of the inference-based querying algorithms in classification aim to choose queries that maximize the expected change in the objective of the inference step~\citep{settles2008multiple,guo2008discriminative} or Fisher information criterion~\citep{Hoi2006,settles2008analysis,hoi2009batch, chaudhuri2015convergence}.
On the other hand, the wide range of studies in prediction-based active learning includes a more varied set of objectives: for instance, the prediction error probability\footnote{Prediction error probability is indeed the frequentist risk function of 0/1 loss, and is also known as \textit{generalization error}.}~\citep{cohn1994improving,freund1997selective,zhu2003combining,nguyen2004active,dasgupta2005coarse,dasgupta2005analysis,balcan2006agnostic,dasgupta2007general,beygelzimer2010agnostic,hanneke2011rates,hanneke2012activized,awasthi2013power,zhang2014beyond}, variance of the predictions~\citep{cohn1996active,schein2007active, ji2012variance}, uncertainty of the learner with respect to the unknown labels as evaluated by the entropy function~\citep{holub2008entropy}, mutual information~\citep{guo2007optimistic,krause2008near,guo2010active,sourati2016classification}, and margin of the samples with respect to the trained hyperplanar discriminant function~\citep{schohn2000less,tong2002support}.

In this manuscript, we focus on the Fisher information criterion used in classification active learning algorithms. These algorithms use a scalar function of the Fisher information matrices computed for parametric models of training and test marginals. In the classification context, this scalar is sometimes called \emph{Fisher information ratio (FIR)}~\citep{settles2008analysis} and its usage is motivated by older attempts in optimal experiment design for statistical regression methods~\citep{fedorov1972theory,mackay1992information,cohn1996neural,fukumizu2000statistical}. 

Among the existing FIR-based classification querying methods, only the very first one proposed by~\citet{Zhang2000} approached the FIR objective from a parameter inference point of view. Using a maximum likelihood estimator (MLE), they claimed (with the proof skipped) that FIR is asymptotically equal to the expectation of the log-likelihood ratio with respect to both test and training samples (see sub-problem~\ref{subproblem:query_selection}). Later on,~\citet{Hoi2006} and~\citet{hoi2009batch}, inspired by~\citet{Zhang2000}, used FIR in connection with a logistic regression classifier with the motivation of decreasing the labels' uncertainty and hence the prediction error. \citet{settles2008analysis} employed this objective with the same motivation, but using a different approximation and optimization technique. More recently, \citet{chaudhuri2015convergence} showed that even finite-sample FIR is closely related to the expected log-likelihood ratio of an MLE-based classifier. However, their results are derived under a different and rather restricting set of conditions and assumptions: they focused on the finite-sample case where the test marginal is a uniform PMF and the proposal marginal is a general PMF (to be determined) over a finite pool of unlabeled samples. Moreover, they assumed that the conditional Fisher information matrix is assumed to be independent of the class labels. Here, in a framework similar to~\citet{Zhang2000} but with a more expanded and different derivation, we discuss a novel theoretical result based on which FIR is related to an MLE-based inference step for a large number of training data. More specifically, under certain regularity conditions required for consistency of MLE and in the absence of model mis-specification, and with no restricting assumptions on the form of test or training marginals, we show that FIR can be viewed as an upper bound for the expected variance of the asymptotic distribution of the log-likelihood ratio. Inspired by \citet{chaudhuri2015convergence}, we also show that under certain extra conditions, this relationship holds even in finite-sample case.

There are two practical issues in employing FIR as a query selection objective: its computation and optimization. First, computing the Fisher information matrices is usually intractable, except for very simple distributions; also FIR depends on the true marginal, which is usually unknown. Therefore, even if the computations are tractable, approximations have to be used for evaluating FIR. Second, the optimization of FIR is straightforward only if a single query is to be selected per iteration, or when the optimization has continuous domain (e.g, optimizing to get the real parameters of the query marginal~\citep{fukumizu2000statistical}). However, the optimization becomes NP-hard when multiple queries are to be selected from a countable set of unlabeled samples (\emph{pool-based batch} active learning). Heuristics have been used to approximate such combinatorial optimization, such as greedy methods~\citep{settles2008analysis} and relaxation to continuous domains~\citep{guo2010active}. Another strategy is to take advantage of \emph{monotonic submodularity} of the objective set functions. If the objective is shown to be monotonically submodular, efficient greedy algorithms can be used for optimization with guaranteed tight bounds~\citep{krause2008near,azimi2012batch,chen2013near}. Regarding FIR, \citet{Hoi2006} proved that, when a logistic regression model is used, a Monte-Carlo simulation of this objective is a monotone and submodular set function in terms of the queries.  

In addition to our theoretical contribution in asymptotically relating FIR to the log-likelihood ratio, we clarify the differences between some of the existing FIR-based querying methods according to the techniques that they use to address the evaluation and optimization issues. Furthermore, we show that monotonicity and submodularity of Monte-Carlo approximation of FIR can be extended from logistic regression models to \emph{any} discriminative classifier.
%
Here is a summary of our contributions in this paper:
\begin{itemize}
\item Establishing a relationship between the Fisher information matrix of the query distribution and the asymptotic distribution of the log-likelihood ratio (section~\ref{subsec:asymptotic_distribution});
\item Showing that FIR can be viewed as an upper bound of the expected asymptotic variance of the log-likelihood ratio, implying that minimizing FIR, as an active learning objective, is asymptotically equivalent to upper-bound minimization of the expected variance of the log-likelihood ratio, as a measure of inference performance (section~\ref{subsec:upper_bound});
\item Proving that under certain assumptions, the above-mentioned asymptotic relationship also holds for finite-sample estimation of FIR (section~\ref{subsubsec:replacing_theta});
\item Discussing different existing methods for coping with practical issues in using FIR in querying algorithms (section~\ref{subsec:practical_issues}), and accordingly providing a unifying framework for existing FIR-based active learning methods (section~\ref{subsec:algorithms}).
\item Proving submodularity for the Monte-Carlo simulation of FIR under \emph{any} discriminative classifier, assuming a pool-based active learning which enables access to approximations of Fisher information matrices of both test and training distributions (Lemma~\ref{lemma:equivalent_set_optimization} and Theorem~\ref{thm:submodularity}).
\end{itemize}

%% file: ClassificationModel.tex
\section{The Framework and Assumptions}
\label{sec:framework_assumptions}
In this paper, we deal with classification problems, where each covariate, represented by a feature vector $\bx$ in vector space $X$, is associated with a numerical class label $y$. Assuming that there are $1<c<\infty$ classes, $y$ can take any integer among the set $\{1,...,c\}$. Suppose that the pairs $(\bx,y)$ are distributed according to a parametric joint distribution $p(\bx,y|\btheta)$, with the parameter space denoted by $\Omega\subseteq\mathbb{R}^d$. Using a set of observed pairs as the training data, $\call_n:=\{(\bx_1,y_1),...,(\bx_n,y_n)\}$, we can estimate $\btheta$ and predict the class labels of the unseen test samples, e.g. by maximizing $p(y|\bx,\btheta)$. In active learning, the algorithm is permitted to take part in designing $\call_n$ by choosing a set of data points $\{\bx_1,...,\bx_n\}$, for which the class labels are then generated using an external oracle.

In addition to the framework described in the last section (see subproblems (i) to (iii)), we make the following assumptions regarding the oracle, our classification model and the underlying data distribution:
\begin{enumerate}
\setcounter{enumi}{-1}
\item\label{item:parameter_dependence} The dependence of the joint distribution to the parameter $\btheta$ comes only from the class-conditional distribution and the marginal distribution does not depend on $\btheta$, that is:
\begin{equation}
\label{eq:type2_distribution}
p(\bx,y|\btheta) \equals p(y|\bx,\btheta) p(\bx).
\end{equation}
\citet{Zhang2000} referred to joint distributions with such parameter dependence as type-II models, as opposed to type-I models which have parameter dependence in both class conditionals and marginal. They argue that active learning is more suitable for type-II models. Moreover, maximizing the joint with respect to the parameter vector in this model, becomes equivalent to maximizing the posterior $p(y|\bx,\btheta)$ (inference step in sub-problem (ii)).
\item\label{item:identifiability} \textit{(Identifiability):} The joint distribution $P_{\btheta}$ (whose density is given by $p(\bx,y|\btheta)$) is identifiable for different parameters. Meaning that for every distinct parameter vectors $\btheta_1$ and $\btheta_2$ in $\Omega$, $P_{\btheta_1}$ and $P_{\btheta_2}$ are also distinct. That is
$$\forall \btheta_1\neq\btheta_2\in\Omega \enskip \exists A\subseteq X\times\{1,...,c\} \quad \mbox{s.t.} \quad P_{\btheta_1}(A)\neq P_{\btheta_2}(A).$$
\item\label{item:support} The joint distribution $P_{\btheta}$ has common support for all $\btheta\in\Omega$.
\item\label{item:oracle_distribution} \textit{(Model Faithfulness): } For any $\bx\in X$, we have access to an oracle that generates a label $y$ according to the conditional $p(y|\bx,\btheta_0)$. That is, the posterior parametric model matches the oracle distribution. We call $\btheta_0$ the true model parameter. 
\item\label{item:iid} \textit{(Training joint): } The set of observations in $\call_n:=\{(\bx_1,y_1),...,(\bx_n,y_n)\}$ are drawn independently from the \emph{training}/\emph{proposal}/\emph{query} joint distribution of the form $p(y|\bx,\btheta_0)q(\bx)$ where $q$ is the training marginal with no dependence on the parameter. 
\item\label{item:test_iid} \textit{(Test joint): } The unseen test pairs are distributed according to the \emph{test}/\emph{true} joint distribution of the form $p(y|\bx,\btheta_0)p(\bx)$ where $p$ is the test marginal with no dependence on the parameter.
\item\label{item:differentiability} \textit{(Differentiability): } The log-conditional $\log p(y|\bx,\btheta)$ is of class $\calc^3(\Omega)$ as a function of $\btheta$ and for all $(\bx,y)\in X\times\{1,...,c\}$ \footnote{We say that a function $f:X\to Y$ is of $\mathcal{C}^p(X)$, for an integer $p>0$, if its derivatives up to $p$-th order exist and are continuous at all points of $X$.}.
\item\label{item:openness}
The parameter space $\Omega$ is compact and there exists an open ball around the true parameter of the model $\btheta_0\in\Omega$.
\item\label{item:FisherInformation} \textit{(Invertibility): } The Fisher information matrix (reviewed in section~\ref{subsec:fisher_information}) of the joint distribution is positive definite and therefore invertible for all $\btheta\in\Omega$, and for any type of marginal that is used under assumption \ref{item:parameter_dependence}. 
\end{enumerate}

\noindent Regarding assumptions \ref{item:iid} and \ref{item:test_iid}, note that the training and test marginals are not necessarily equal. The test marginal is usually not known beforehand and $q$ cannot be set equal to $p$ in practice, hence $q$ can be viewed as a proposal distribution. Such inconsistency is what \citet{shimodaira2000improving} called \emph{covariate shift in distribution}. In the remaining sections of the report, we use subscripts $p$ and $q$ for the statistical operators that consider $p(\bx)$ and $q(\bx)$ as the marginal in the joint distribution, respectively. We explicitly mention $\bx$ as the input argument in order to refer to marginal operators. For instance, $\mathbb{E}_q$ denotes the joint expectation with respect to $q(\bx)p(y|\btheta,\bx)$, whereas $\mathbb{E}_{q(\bx)}$ denotes the marginal expectation with respect to $q(\bx)$.

\section{Background}
\label{sec:background}
Here, we provide a short review of maximum likelihood estimation (MLE) as our inference method, and briefly introduce Fisher information of a parametric distribution. These two basic concepts enable us to explain some of the key properties of MLE, upon which our further analysis of FIR objective relies.  Note that our focus in this section is on sub-problem~\ref{subproblem:passive_learning} with the assumptions listed above.

\subsection{Maximum Likelihood Estimation}
\label{subsec:MLE}
In this section, we review maximum likelihood estimation in the context of classification problem.
Given a training data set $\call_n=\{(\bx_1,y_1),...,(\bx_n,y_n)\}$, a maximum likelihood estimate (MLE) is obtained by maximizing the log-likelihood function over all pairs inside $\call_n$, with respect to the parameter $\btheta$:
\begin{align}
\label{eq:MLE}
\hat{\btheta}_n \equals \operatorname*{arg\,max}_{\btheta} \enskip \log p\left(\call_n|\btheta \right),
\end{align}
Under the assumptions \ref{item:parameter_dependence} and \ref{item:iid}, the optimization in~(\ref{eq:MLE}) can be written as
\begin{align}
\label{eq:MLE_conditional}
\hat{\btheta}_n \equals \operatorname*{arg\,max}_{\btheta} \enskip \sum_{i=1}^n \log p(y_i|\bx_i,\btheta),
\end{align}
Equation~(\ref{eq:MLE_conditional}) shows that MLE does not depend on the marginal when using type-II model. Hence, in our analysis we focus on the conditional log-likelihood as the classification objective, and simply call it the log-likelihood function when viewed as a function of the parameter vector $\btheta$, for any given pair $(\bx,y)\in X\times\{1,...,c\}$:
\begin{equation}
\ell(\btheta;\bx,y) \enskip := \enskip \log p(y|\bx,\btheta).
\end{equation}
Moreover, for any set of pairs independently generated from the joint distribution of the training data, such as $\call_n$ mentioned in~\ref{item:iid}, the log-likelihood function will be:
\begin{equation}
\label{eq:loglikelihood_function}
\ell(\btheta;\call_n) \equals \sum_{i=1}^n \ell(\btheta;\bx_i,y_i)\equals \sum_{i=1}^n \log p(y_i|\bx_i,\btheta).
\end{equation}
hence the MLE can be rewritten as
\begin{equation}
\label{eq:MLE_ell}
\hat{\btheta}_n \equals \operatorname*{arg\,max}_{\btheta} \enskip \sum_{i=1}^n \ell(\btheta;\bx_i,y_i).
\end{equation}
Doing this maximization usually involves the computation of the stationary points of the log-likelihood, which requires calculating $\nabla_{\btheta}\ell(\btheta;\call_n)=\sum_{i=1}^n\nabla_{\btheta}\ell(\btheta;\bx_i,y_i)$.
For models assumed in~\ref{item:parameter_dependence}, each of the derivations in the summation is equal to the \emph{score function} defined as the gradient of the joint log-likelihood: 
\begin{equation}
\label{eq:score_loglikelihood}
\nabla_{\btheta}\ell(\btheta;\bx,y) \equals  \nabla_{\btheta}\log p(y|\bx,\btheta) \equals \nabla_{\btheta}\log p(\bx,y|\btheta) ,
\end{equation}
Equation~(\ref{eq:score_loglikelihood}) implies that the score will be the same no matter whether we choose the training or test distribution as our marginal. Furthermore, under regularity conditions~\ref{item:differentiability}, the score is always a zero-mean random variable\footnote{Score function is actually zero-mean even under weaker regularity conditions.}.

Finally, using MLE to estimate $\hat{\btheta}_n$, class label of a test sample $\bx$ will be predicted as the class with the highest log-likelihood value:
\begin{equation}
\hat{y}(\bx)  \equals \operatorname*{arg\,max}_{y} \ell(\hat{\btheta}_n;\bx,y).
\end{equation}

\subsection{Fisher Information}
\label{subsec:fisher_information}
Here we give a very short introduction to Fisher information. More detailed descriptions about this well-known criterion can be found in various textbooks, such as~\citet{lehmann1998theory}. 

Fisher information of a parametric distribution is a measure of information that the samples generated from that distribution provide regarding the parameter. It owes part of its importance to the Cram\'{e}r-Rao Theorem (see Appendix~\ref{subsec:parameter_estimation}, Theorem~\ref{thm:CR_bound}), which guarantees a lower-bound for the covariance of the parameter estimators.
 
Fisher information, denoted by $\bI(\btheta)$, is defined as the expected value of the outer-product of the score function with itself, evaluated at some $\btheta\in\Omega$. In our classification context, taking the expectation with respect to the training or test distributions gives us the training or test Fisher information criteria, respectively:
\begin{align}
\label{eq:FI_training_test}
\bI_q(\btheta) \enskip &:= \enskip \mathbb{E}_q\left[ \nabla_{\btheta}\log p(\bx,y|\btheta)\cdot\nabla_{\btheta}^\top\log p(\bx,y|\btheta) \right] \nonumber\\
\bI_p(\btheta) \enskip &:= \enskip \mathbb{E}_p\left[ \nabla_{\btheta}\log p(\bx,y|\btheta)\cdot\nabla_{\btheta}^\top\log p(\bx,y|\btheta) \right]
\end{align}
Here, we focus on $\bI_q$ to further explain Fisher information criterion. Our descriptions here can be directly generalized to $\bI_p$ as well. First, note that from equation~(\ref{eq:score_loglikelihood}) and that the score function is always zero-mean, one can reformulate the definition as:
\begin{align}
\label{eq:fisher_information_loglikelihood}
\bI_q(\btheta) &\equals \mathbb{E}_q\left[ \nabla_{\btheta}\ell(\btheta;\bx,y)\cdot \nabla_{\btheta}^\top\ell(\btheta;\bx,y) \right] \nonumber\\
&\equals \mbox{Cov}_q\left[\nabla_{\btheta}\ell(\btheta;\bx,y) \right]
\end{align}
Under the differentiability conditions~\ref{item:differentiability}, it is easy to show that we can also write the Fisher information in terms of the Hessian matrix of the log-likelihood:
\begin{equation}
\label{eq:fisher_hessian}
\bI_q(\btheta) \equals -\mathbb{E}_q\left[\nabla_{\btheta}^2\ell(\btheta;\bx,y) \right]
\end{equation}
Recall that the subscript $q$ in equations~(\ref{eq:fisher_information_loglikelihood}) and~(\ref{eq:fisher_hessian}) indicates that the expectations are taken with respect to the joint distribution that uses $q(\bx)$ as the marginal, that is $p(\bx,y|\btheta)=q(\bx)p(y|\bx,\btheta)$. Expansion of the expectation in~(\ref{eq:fisher_hessian}) results
\begin{align}
\label{eq:fisher_information_integration}
\bI_q(\btheta) &\equals -\mathbb{E}_{q(\bx)}\left[ \mathbb{E}_{y|\bx,\btheta}\left[ \nabla_{\btheta}^2\ell(\btheta;\bx,y) | \bx,\btheta \right] \right] \nonumber\\
&\equals -\int_{\bx\in X} q(\bx) \left[ \sum_{y=1}^c p(y|\bx,\btheta) \cdot \nabla^2_{\btheta}\ell(\btheta;\bx,y) \right]d\bx
\end{align}

\subsection{Some Properties of MLE}
\label{subsec:properties_MLE}
In this section, we formalize some of the key properties of MLE, which make this estimator popular in various fields.
They are also very useful in the theoretical analysis of FIR, provided in the next section. More detailed descriptions of these properties, together with the proofs that are skipped here, can be found in different sources, such as~\citet{wasserman2004all} and~\citet{lehmann1998theory}.

Note that a full understanding of the properties described in this section requires the knowledge of different modes of statistical convergence, specifically, convergence in probability ($\overset{P}{\to}$), and convergence in law ($\overset{L}{\to}$). A brief overview of these concepts are given in Appendix~\ref{sec:statistical_background}.

\begin{theorem}[\citet{lehmann1998theory}, Theorem 5.1]
\label{thm:MLE_consistency}
If the assumptions~\ref{item:parameter_dependence} to~\ref{item:openness} hold, then there exists a sequence of solutions $\left\{\hat{\btheta}^*_n\right\}_{n=1}^\infty$ to $\nabla_{\btheta}\ell(\btheta;\call_n)=0$ that converges to the true parameter $\btheta_0$ in probability.
\end{theorem}
Note that Theorem~\ref{thm:MLE_consistency} does not imply that convergence holds for \emph{any} sequence of MLEs. Hence, if there are multiple solutions to equation $\nabla_{\btheta}\ell(\btheta;\call_n)=0$ (the equation to solve for finding the stationary points) for every $n$, it is not obvious which root to select as $\hat{\btheta}_n^*$ to sustain the convergence. Therefore, while consistency of the MLE is guaranteed for models with a unique root of the score function evaluated at $\call_n$, it is not trivial how to build a consistent sequence when multiple roots exist. Here, in order to remove this ambiguity, we assume that either the roots are unique, become asymptotically unique, or we have access to an external procedure guiding us to select the proper roots so that $\hat{\btheta}_n\overset{P}{\to}\btheta_0$. We will denote the selected roots the same as $\hat{\btheta}_n$ from now on.

\begin{theorem}[\citet{lehmann1998theory}, Theorem 5.1]
\label{thm:asymptotic_efficiency_MLE}
Let $\hat{\btheta}_n$ be the maximum likelihood estimator based on the training data set $\call_n$. If the assumptions~\ref{item:parameter_dependence} to~\ref{item:FisherInformation} hold, then the MLE $\hat{\btheta}_n$ has a zero-mean normal asymptotic distribution with the covariance equal to the inverse Fisher information matrix, and with the convergence rate of $1/2$:
\begin{equation}
\label{eq:MLE_convergence_in_law}
\sqrt{n}(\hat{\btheta}_n-\btheta_0)\enskip \overset{L}{\rightarrow} \enskip \caln\left( \mathbf{0},\bI_{q}(\btheta_0)^{-1} \right)
\end{equation}
\end{theorem}
Theorems~\ref{thm:asymptotic_efficiency_MLE} and Cram\'{e}r-Rao bound (see Appendix~\ref{sec:statistical_background}), together with the consistency assumption, i.e. $\hat{\btheta}_n\overset{P}{\to}\btheta_0$, imply that MLE is an asymptotically efficient estimator with the efficiency equal to the training Fisher information. One can rewrite~(\ref{eq:MLE_convergence_in_law}) as
\begin{equation}
\sqrt{n}\cdot\bI_q(\btheta_0)^{1/2}(\hat{\btheta}_n-\btheta_0) \enskip \overset{L}{\to}\enskip \caln(\mathbf{0},\mathbb{I}_d)
\end{equation}
In the following corollary, we see that if we substitute $\bI_q(\btheta_0)$ with $\bI_q(\hat{\btheta}_n)$, the new sequence still converges to a normal distribution:

\begin{corollary}[\citet{wasserman2004all}, Theorem 9.18]
\label{cor:MLE_convergence_in_law_thetahat}
Under the assumptions of Theorem~\ref{thm:asymptotic_efficiency_MLE}, we get
\begin{equation}
\label{eq:MLE_convergence_in_law_thetahat}
\sqrt{n}\cdot\bI_{q}(\hat{\btheta}_n)^{1/2}(\hat{\btheta}_n-\btheta_0)\enskip \overset{L}{\rightarrow} \enskip \caln\left( \mathbf{0},\mathbb{I}_d \right)
\end{equation}
\end{corollary}

%% file: TheoreticalAnalysis.tex
In this section, we give our main theoretical analysis to relate FIR to the asymptotic distribution of the parameter log-likelihood ratio. Using the established relationship, we then show that FIR can be viewed as an asymptotic upper-bound of the expected variance of the loss function.

\subsection{Asymptotic Distribution of MLE-based Classifier}
\label{subsec:asymptotic_distribution}
Recall that the estimated parameter $\hat{\btheta}_n$ is obtained from a given proposal distribution $q(\bx)$. The log-likelihood ratio function, at a given pair $(\bx,y)$, is defined as:
\begin{equation}
\label{eq:likelihood_ratio}
\ell(\hat{\btheta}_n;\bx,y) \enskip - \enskip \ell(\btheta_0;\bx,y).
\end{equation}
This ratio can be viewed as an example of the classification loss function whose expectation with respect to the test joint distribution of $\bx$ and $y$, results in the \emph{discrepancy} between the true conditional $p(y|\bx,\btheta_0)$ and MLE conditional $p(y|\bx,\hat{\btheta}_n)$~\citep{murata1994network}. Here, we analyze this measure asymptotically as $(n\to\infty)$. Primarily, note that based on continuity of the log-likelihood function \ref{item:differentiability} and consistency of MLE (Theorem~\ref{thm:MLE_consistency}), equation~(\ref{eq:likelihood_ratio}) converges in probability to zero for any $(\bx,y)$.  

Furthermore, equation~(\ref{eq:likelihood_ratio}) is dependent on both the true marginal $p(\bx)$ (through the test pairs, where it should be evaluated) and the proposal marginal $q(\bx)$ (through the MLE $\hat{\btheta}_n$). In the classification context, \citet{Zhang2000} claimed that the expected value of this ratio with respect to both marginals converges to $\mbox{tr}[\bI_q(\btheta_0)^{-1}\bI_p(\btheta_0)]$ with the convergence rate equal to unity. In the scalar case, $\mbox{tr}[\bI_q(\btheta_0)^{-1}\bI_p(\btheta_0)]$ is equal to the ratio of the Fisher information of the true and proposal distributions, the reason why it is sometimes referred to as the \emph{Fisher information ratio}~\citep{settles2008analysis}. This objective have been widely studied in linear and non-linear regression problems~\citep{fedorov1972theory,mackay1992information,murata1994network,cohn1996neural,fukumizu2000statistical}. However, it is not as fully analyzed in classification.

\citet{Zhang2000} and many papers following them \citep{Hoi2006,settles2008analysis,hoi2009batch}, used this function as an \emph{asymptotic} objective in active learning to be optimized with respect to the proposal $q$. Here, we show that this objective can also be viewed as an \emph{upper bound} for the expected variance of the asymptotic distribution of~(\ref{eq:likelihood_ratio}).

First, we investigate the asymptotic distribution of the log-likelihood ratio in two different cases:
\begin{theorem}
\label{thm:asymptotic_distribution}
If the assumptions~\ref{item:parameter_dependence} to~\ref{item:FisherInformation} hold, then, at any given $(\bx,y)\in X\times\{1,...,c\}$:
\begin{enumerate}[label=(\Roman*)]
\item In case $\nabla_{\btheta}\ell(\btheta_0;\bx,y)\neq\mathbf{0}$, the log-likelihood ratio follows an asymptotic normality with convergence rate equal to $1/2$. More specifically
\begin{equation}
\label{eq:asymptotic_distribution}
\sqrt{n}\cdot\bigg(\ell(\hat{\btheta}_n;\bx,y) -  \ell(\btheta_0;\bx,y)\bigg) \enskip \overset{L}{\to} \enskip \caln\bigg(0, \mbox{\normalfont{tr}}\big[\nabla_{\btheta}\ell(\btheta_0;\bx,y)\cdot\nabla_{\btheta}^\top\ell(\btheta_0;\bx,y)\cdot\bI_q(\btheta_0)^{-1}\big]\bigg).
\end{equation}
\item In case $\nabla_{\btheta}\ell(\btheta_0;\bx,y)=\mathbf{0}$ and $\nabla_{\btheta}^2\ell(\btheta_0;\bx,y)$ is non-singular, the asymptotic distribution of the log-likelihood ratio is a mixture of first-order Chi-square distributions, and the convergence rate is one. More specifically:
\begin{equation}
\label{eq:asymptotic_distribution_zeroder}
n\cdot \bigg(\ell(\hat{\btheta}_n;\bx,y) \enskip - \enskip \ell(\btheta_0;\bx,y)\bigg) \enskip \overset{L}{\to} \enskip \sum_{i=1}^d \lambda_{i}\cdot\chi_1^2 
\end{equation}
where $\lambda_i$'s are eigenvalues of $\bI_q(\btheta_0)^{-1/2}\nabla_{\btheta}^2\ell(\btheta_0;\bx,y)\bI_q(\btheta_0)^{-1/2}$.
\end{enumerate}
\end{theorem}
\begin{proof}
Due to assumptions~\ref{item:parameter_dependence} to~\ref{item:openness}, Theorem~\ref{thm:asymptotic_efficiency_MLE} holds and therefore we have $\sqrt{n}\cdot(\hat{\btheta}_n - \btheta_0)\overset{L}{\to}\caln(\mathbf{0},\bI_q(\btheta_0)^{-1})$. The rest of the proof is based on the Delta method in the two modes described in Appendix~\ref{sec:statistical_background} (Theorems~\ref{thm:multivariate_delta_method} and~\ref{thm:delta_method_zeroder}):
\begin{enumerate}[label=(\Roman*)]
\item $\nabla_{\btheta}\ell(\btheta_0;\bx,y)\neq\mathbf{0}$ :

Since the expected log-likelihood function, evaluated at a given pair $(\bx,y)$, is assumed to be continuously differentiable \ref{item:differentiability} and that $\nabla_{\btheta}\ell(\btheta_0;\bx,y)\neq\mathbf{0}$, we can apply Theorem~\ref{thm:multivariate_delta_method} to $\ell(\hat{\btheta}_n;\bx,y) - \ell(\btheta_0;\bx,y)$ to write:
\begin{equation}
\sqrt{n}\cdot \bigg(\ell(\hat{\btheta}_n;\bx,y) \enskip - \enskip \ell(\btheta_0;\bx,y) \bigg) \enskip\overset{L}{\to}\enskip \caln\bigg(0 \enskip,\enskip \nabla_{\btheta}^\top\ell(\btheta_0;\bx,y)\cdot\bI_q(\btheta_0)^{-1}\cdot \nabla_{\btheta}\ell(\btheta_0;\bx,y) \bigg),
\end{equation}
where the scalar variance can also be written in a trace format.
\item $\nabla_{\btheta}\ell(\btheta_0;\bx,y)=\mathbf{0}$ and $\nabla_{\btheta}^2\ell(\btheta_0;\bx,y)$ non-singular :

In this case, the conditions in Theorem~\ref{thm:delta_method_zeroder} are satisfied (with $\bSigma=\bI_q(\btheta_0)^{-1}$ and $g(\btheta)=\ell(\btheta;\bx,y)$), and therefore we can directly write~(\ref{eq:asymptotic_distribution_zeroder}) from equations~(\ref{eq:delta_method_zeroder}).
\end{enumerate}
\end{proof}
Theorem~\ref{thm:asymptotic_distribution} regards the log-likelihood ratio~(\ref{eq:likelihood_ratio}) evaluated at any arbitrary pair $(\bx,y)$. Note that if we consider the training pairs in $\call_n$, which are used to obtain $\hat{\btheta}_n$, it is known that the ratio evaluated at the training set converges to a single first-degree Chi-square distribution, that is
\begin{equation}
\label{eq:asymptotic_distribution_training}
\ell(\hat{\btheta}_n;\call_n) \enskip - \enskip \ell(\btheta_0;\call_n) \enskip \overset{L}{\to} \enskip \frac{1}{2} \chi_1^2
\end{equation}

Theorem~\ref{thm:asymptotic_distribution} implies that variance of the asymptotic distribution of the log-likelihood ratio in case (I) is $\mbox{\normalfont{tr}}\big[\nabla_{\btheta}\ell(\btheta_0;\bx,y)\cdot\nabla_{\btheta}^\top\ell(\btheta_0;\bx,y)\cdot\bI_q(\btheta_0)^{-1}\big]$, whereas in case (II), from Theorem~\ref{thm:delta_method_zeroder} (see Appendix~\ref{sec:statistical_background}), the variance is $\frac{1}{2}\left\|\bI_q(\btheta_0)^{-1/2}\nabla_{\btheta}^2\ell(\btheta_0;\bx,y)\bI_q(\hat{\btheta}_n)^{-1/2} \right\|_F^2$. Therefore, it is evident that the variance of the log-likelihood ratio at any $(\bx,y)$ is reciprocally dependent on the training Fisher information. From this point of view, one can set the training distribution such that it leads to a Fisher information that minimizes this variance. Unless the parameter and hence the Fisher information is univariate, it is not clear what objective to optimize with respect to $q$ such that the resulting Fisher information minimizes the variance. 
%
%
%
\subsection{Establishing the Upper Bound}
\label{subsec:upper_bound}
In the next theorem, we show that the Fisher information ratio 
, $\mbox{tr}\left[ \bI_q(\btheta_0)^{-1}\bI_p(\btheta_0)\right]$, is a reasonable candidate objective to minimize in order to get a training distribution $q$ for the multivariate case:
\begin{theorem}
\label{thm:FIR_variance_upper_bound}
If the assumptions~\ref{item:parameter_dependence} to~\ref{item:FisherInformation} hold, then:
\begin{equation}
\label{eq:FIR_variance_upper_bound}
\mathbb{E}_p\left[\mbox{\normalfont Var}_q\left(\lim_{n\to\infty}\sqrt{n}\cdot[\ell(\hat{\btheta}_n;\bx,y)-\ell(\btheta_0;\bx,y)]\right) \right] \enskip \leq \enskip \mbox{\normalfont tr}\bigg[\bI_q(\btheta_0)^{-1}\bI_p(\btheta_0)\bigg].
\end{equation}
The equality holds when the set of pairs $(\bx,y)$ where we have zero score function at $\btheta_0$, i.e. $\nabla_{\btheta}\ell(\btheta_0;\bx,y)=\mathbf{0}$, has measure zero under the true joint distribution $P_{\btheta_0}$ in $X\times\{1,...,c\}$.
\end{theorem}
\begin{proof}
Note that, from Theorem~\ref{thm:asymptotic_distribution}, when $\nabla_{\btheta}\ell(\btheta_0;\bx,y)=0$ the convergence rate of the log-likelihood ratio is one and therefore it is of $O_p\left(\frac{1}{n}\right)$. Thus, in this case we have $\sqrt{n}\cdot[\ell(\hat{\btheta}_n;\bx,y)-\ell(\btheta_0;\bx,y)]=O_p\left(\frac{1}{\sqrt{n}}\right)$ and it converges to zero in probability (and in law). Now, define the region $R_0\subseteq X\times\{1,...,c\}$ by
\begin{equation}
R_0 \enskip := \enskip \left\{(\bx,y)| \nabla_{\btheta}\ell(\btheta_0;\bx,y)=0 \right\}
\end{equation}
Variance of the asymptotic distribution of $\sqrt{n}\cdot[\ell(\hat{\btheta}_n;\bx,y)-\ell(\btheta_0;\bx,y)]$, considering both cases $\nabla_{\btheta}\ell(\btheta_0;\bx,y)=0$ (with probability $P_{\btheta_0}(R_0)$) and $\nabla_{\btheta}\ell(\btheta_0;\bx,y)\neq0$ (with probability $1-P_{\btheta_0}(R_0)$), can be written as:
\begin{align}
\label{eq:FIR_variance_upper_bound_noexp}
\mbox{Var}&\left(\lim_{n\to\infty}\sqrt{n}\cdot[\ell(\hat{\btheta}_n;\bx,y)-\ell(\btheta_0;\bx,y)]\right) \nonumber \\
&\equals [1-P_{\btheta_0}(R_0)]\cdot \mbox{\normalfont tr}\big[\nabla_{\btheta}\ell(\btheta_0;\bx,y)\cdot\nabla_{\btheta}^\top\ell(\btheta_0;\bx,y)\cdot\bI_q(\btheta_0)^{-1}\big] \enplus P_{\btheta_0}(R_0)\cdot0 \nonumber\\ 
&\leq \enskip \mbox{\normalfont tr}\big[\nabla_{\btheta}\ell(\btheta_0;\bx,y)\cdot\nabla_{\btheta}^\top\ell(\btheta_0;\bx,y)\cdot\bI_q(\btheta_0)^{-1}\big] 
\end{align}
Taking the expectation of both sides with respect to the true joint, gives the inequality~(\ref{eq:FIR_variance_upper_bound}). If the set of pairs $(\bx,y)$ where $\nabla_{\btheta}\ell(\btheta_0;\bx,y)=\mathbf{0}$ form a zero measure set under $P_{\btheta_0}$, then $P_{\btheta_0}(R_0)=0$ and we get equality in~(\ref{eq:FIR_variance_upper_bound_noexp}) and hence an equality in~(\ref{eq:FIR_variance_upper_bound}).
\end{proof}
Theorem~\ref{thm:FIR_variance_upper_bound} implies that minimizing the Fisher information ratio with respect to $q$, is indeed the upper-bound minimization of the expected variance of the asymptotic distribution of the log-likelihood ratio.

%% file: InPractice.tex
In this section, we explain how inequality~(\ref{eq:FIR_variance_upper_bound}) can be utilized in practice as an objective function for active learning. The left-hand-side is the objective that is more reasonable to minimize from classification point of view. However, its optimization is intractable and FIR-based methods approximate it by its upper-bound minimization. Querying can be done with this objective by first learning the optimal proposal distribution $q$ that minimizes the left-hand-side of inequality~(\ref{eq:FIR_variance_upper_bound}) and then drawing the queries from this optimal distribution:
\begin{subequations}
\label{eq:FIRAL}
    \begin{align}
    \label{eq:q_optimization} 
    q^* \enskip&=\enskip \argmin_{q}\enskip \mbox{tr}[\bI_q(\btheta_0)^{-1}\bI_p(\btheta_0)] \\
    \label{eq:Xq_sampling}
    X_q \enskip&\sim\enskip q^*(\bx)
    \end{align}
\end{subequations}
where $X_q$ is the set of queries whose samples are drawn from $q^*$. Note that in~(\ref{eq:FIRAL}), due to the sampling process, $X_q$ cannot be deterministically determined even by fixing all the other parameters leading to a fixed query distribution $q^*$ (ignoring the uncertainties in the numerical optimization processes). Hence this setting is sometimes called \emph{probabilistic} active learning. Notice that in pool active learning,  $q$ should be constrained to be a PMF over the unlabeled pool from which the queries are to be chosen. Relaxing $q$ to continuous distributions leads to \emph{synthetic} active learning, since each time an unseen sample will be \emph{synthesized} by sampling from $q$. We will see later that in some pool-based applications, the \emph{objective functional} of $q$ is approximated as a \emph{set function} of $X_q$, and therefore a combinatorial optimization is performed directly with respect to the query set. 

As mentioned, $q^*$ is an upper-bound minimization of the expected asymptotic loss variance. Moreover, there are a number of unknown variables involved in FIR objective, such as $p$ and $\btheta_0$. In practice, estimations of these unknown variables are used in the optimization process for active learning. Therefore, although the derivations in the previous section (Theorem~\ref{thm:FIR_variance_upper_bound}) are made based on one querying of infinitely many samples, in active learning a finite sample approximation of the cost function is used in an iterative querying process. As the  number of querying iterations in active learning increases, the parameter estimates get more accurate and so does the approximate FIR objective. In the next section, we show that under certain assumptions the optimization with respect to proposal distribution in each iteration is yet another upper-bound minimization similar to~(\ref{eq:FIR_variance_upper_bound}). More specifically, Remark 6 (see Section~\ref{subsubsec:replacing_theta}) shows that although the proposal distribution is optimized separately in each iteration of an FIR-based active learning algorithm, minimizing the approximate FIR at each iteration is still an upper-bound minimization of the original cost function (i.e. left-hand-side of~(\ref{eq:FIR_variance_upper_bound})).

Algorithm~\ref{alg:AL} shows steps of a general discriminative classification with active learning. We assume an initial training set $\call_{n_0}=\{(\bx_1,y_1),...,(\bx_{n_0},y_{n_0}) \}$ is given based on which an initial MLE $\hat{\btheta}_{n_0}$ can be obtained. The initial MLE enables us to approximate the active learning objective function and therefore select queries for building the new training set. After obtaining the query set $X_q$, for each individual sample $\bx\in X_q$, we request its labels $y(\bx)$ from the oracle (or equivalently, sample it from the true conditional, $y(\bx)\sim p(y|\bx,\btheta_0)$). These pairs are then added into the training set to get $\call_{n_1}$, which in turn, is used to update the MLE to $\hat{\btheta}_{n_1}$. Size of the new training data is $n_1=n_0 + |X_q|$. This procedure can be done repeatedly for a desirable number of iterations $i_{max}$. All different techniques that we discuss in this section, differ only in line~\ref{line:generate_queries} and the rest of the steps are common between them. Each active learning algorithm $\cala$ takes the current estimate of the parameter $\hat{\btheta}_{n_{i-1}}$ possibl together with the unlabeled set of samples $X_p$, and generate a set of queries $X_q$ to be labeled for the next iteration.

\setcounter{algocf}{-1}
\begin{algorithm}[!t]
\caption{Classification with Active Learning}
\label{alg:AL}
\textbf{Inputs:} The initial training set $\call_{n_0}$; number of querying iterations $i_{max}$\\
\textbf{Outputs:} The trained classifier with MLE $\hat{\btheta}_{n_{i_{max}}}$\\[-5pt]
\noindent\rule{.5\columnwidth}{.5pt}\\
\SetNlSkip{1em}
\DontPrintSemicolon
\tcc{\footnotesize Initializations}
\nl$\hat{\btheta}_{n_0}\gets\argmax_{\btheta}\ell(\btheta;\call_{n_0})$\;
\tcc{\footnotesize  Starting the Iterations}
\nl\For{$i=1\to i_{max}$}{
\tcc{\footnotesize Generating the query set by optimizing a querying objective}
\nl$X_q \enskip\gets\enskip \cala(\call_{n_{i-1}}, \hat{\btheta}_{n_{i-1}})$ \label{line:generate_queries}\;
\tcc{\footnotesize Request the queries' labels from the oracle}
\nl$y(\bx) \enskip \sim\enskip p(y|\bx,\btheta_0) \quad \forall\bx\in X_q$\;
\tcc{\footnotesize Taking care of indexing}
\nl$n_i \enskip \gets \enskip n_{i-1} + |X_q|$\;
\tcc{\footnotesize Update the training set and update MLE}
\nl$\call_{n_i}\gets\call_{n_{i-1}}\cup\left\{\bigcup_{\bx\in X_q}(\bx,y(\bx))\right\}$\;
\nl$\hat{\btheta}_{n_i}\gets \argmax_{\btheta}\ell(\btheta;\call_{n_i})$\;
}
\nl\textbf{return} $\hat{\btheta}_{n_{i_{max}}}$\;
\end{algorithm}

In our analysis in the subsequent sections, we focus on a specific \emph{querying iteration} indexed by $i$ (as a positive integer). For simplicity, we replace $n_{i-1}$ and $n_i$ (size of the training data set before and after iteration $i$) by $n'$ and $n$, respectively. Hence, iteration $i$ consists of using the available parameter estimate, $\hat{\btheta}_{n'}$ obtained through the current training data set $\call_{n'}$, to generate queries using a given querying algorithm $\cala(\hat{\btheta}_{n'},\call_{n'})$ and then update the classifier's parameter estimate accordingly to $\hat{\btheta}_n$. 

In what follows, we first discuss practical issues in using FIR in query selection (section~\ref{subsec:practical_issues}) and then review existing algorithms based on this objective (section~\ref{subsec:algorithms}).

\subsection{Practical Issues}
\label{subsec:practical_issues}
The main difficulties consist of (1) having unknown variables in the objective, such as the test marginal, $p(\bx)$, and the true parameter, $\btheta_0$, and (2) lack of closed form for Fisher information matrices for most cases. In the next two sections, we review different hacks and solutions that have been proposed to resolve these issues.

\subsubsection{Replacing $\btheta_0$ by $\hat{\btheta}_{n'}$}
\label{subsubsec:replacing_theta}
Since $\btheta_0$ is not known, the simplest idea is to replace it by the current parameter estimate, that is $\hat{\btheta}_{n'}$~\citep{fukumizu2000statistical,settles2008analysis,Hoi2006,hoi2009batch,chaudhuri2015convergence}. Clearly, as the algorithm keeps running the iterations ($n'$ increases), the approximate objective (which contains $\btheta_{n'}$ instead of $\btheta_0$) gets closer to the original objective. This is due to the regularity and invertibility conditions assumed for the log-likelihood function and Fisher information matrices, respectively. Moreover, \citet{chaudhuri2015convergence} analyzed how this approximation effects the querying performance in finite-sample case. 

Their analysis is done only for pool-based active learning, and when the test marginal $p(\bx)$ is a \emph{uniform distribution} $U(\bx)$ over the pool $X_p$. It is also assumed that the Hessian $\frac{\partial^2 \ell(\btheta;\bx,y)}{\partial \btheta^2}$ is independent of the class labels $y$, and therefore can be viewed as the conditional Fisher information $\bI(\btheta,\bx)$ (that is $\bI_p(\btheta)=\mathbb{E}_{p(\bx)}[\bI(\btheta,\bx)]$). Furthermore, there assumed to exist four positive constants $L_1,L_2,L_3, L_4\geq0$ such that the following four inequalities hold for all $\bx\in X_p$, $y\in\{1,...,c\}$ and $\btheta\in\Theta$:
\begin{align}
\label{eq:Chaudhuri_assumptions}
\nabla\ell(\btheta_0;\bx,y)^\top\bI_p(\btheta_0)^{-1}\nabla\ell(\btheta_0;\bx,y) \enskip&\leq\enskip L_1 \\
\left\|\bI_p(\btheta_0)^{-1/2}\bI(\btheta_0,\bx)\bI_p(\btheta_0)^{-1/2} \right\| \enskip &\leq \enskip L_2 \nonumber\\
\left\|\bI_p(\btheta_0)^{-1/2}(\bI(\btheta',\bx)-\bI(\btheta'',\bx))\bI_p(\btheta_0)^{-1/2} \right\| \enskip&\leq\enskip L_3(\btheta'-\btheta'')^\top\bI_p(\btheta_0)(\btheta'-\btheta'') \nonumber\\
-L_4\|\btheta - \btheta_0\|_2\bI(\btheta_0,\bx) \enskip\preceq\enskip \bI(\btheta,\bx)-\bI(\btheta_0,\bx) \enskip&\preceq\enskip L_4\|\btheta-\btheta_0\|_2\bI(\btheta_0,\bx) \nonumber
\end{align}
where $\btheta'$ and $\btheta''$ are any two parameters in a fixed neighborhood of $\btheta_0$. Then, provided that $n'$ is large enough, the following remark can be shown regarding the relationship between the FIRs computed at $\btheta_0$ and an estimate $\hat{\btheta}_{n'}$:

\begin{remark}
Let the assumptions~(\ref{item:parameter_dependence}) to~(\ref{item:FisherInformation}) and those in~(\ref{eq:Chaudhuri_assumptions}) hold. Moreover, assume that the Hessian is independent of the class labels. If $n'$ is large enough, then the following inequality holds for any $\beta\geq10$ with high probability:
\begin{equation}
\label{eq:replacement_inequality}
\mbox{\normalfont tr}\left[\bI_q(\btheta_0)^{-1}\bI_p(\btheta_0)\right] \enskip \leq \enskip \frac{\beta+1}{\beta-1}\cdot\mbox{\normalfont tr}\left[\bI_q(\hat{\btheta}_{n'})^{-1}\bI_p(\hat{\btheta}_{n'})\right]
\end{equation}
The minimum value for $n'$ that is necessary for having this inequality with probability $1-\delta$, increases quadratically with $\beta$ and reciprocally with $\delta$~\citep[Lemma 2]{chaudhuri2015convergence}.
\end{remark}
\begin{proof}
It is shown in the proof of Lemma 2 in \citet{chaudhuri2015arxiv} that under assumptions mentioned in the statement, the following inequalities hold with probability $1-\delta$:
\begin{equation}
\frac{\beta-1}{\beta}\bI(\bx,\btheta_0) \enskip\preceq\enskip \bI(\bx,\hat{\btheta}_{n'}) \enskip\preceq\enskip \frac{\beta+1}{\beta}\bI(\bx,\btheta_0).
\end{equation}
Taking expectation with respect to $p(\bx)$ and $q(\bx)$ result:
\begin{subequations}
\begin{align}
\frac{\beta-1}{\beta}\bI_p(\btheta_0) \enskip\preceq\enskip \bI_p(\hat{\btheta}_{n'}) \enskip\preceq\enskip \frac{\beta+1}{\beta}\bI_p(\btheta_0) \label{ineq:p_expectation}\\
\frac{\beta-1}{\beta}\bI_q(\btheta_0) \enskip\preceq\enskip \bI_q(\hat{\btheta}_{n'}) \enskip\preceq\enskip \frac{\beta+1}{\beta}\bI_q(\btheta_0) \label{ineq:q_expectation}
\end{align}
\end{subequations}
Since $\bI_q(\btheta_0)$ and $\bI_q(\hat{\btheta}_{n'})$ are assumed to be positive definite, we can write~(\ref{ineq:q_expectation}) in terms of inverted matrices\footnote{For any two positive definite matrices $\mathbf{A}$ and $\mathbf{B}$, we have that $\mathbf{A}\succeq\mathbf{B} \enskip\Rightarrow\enskip \mathbf{A}^{-1}\preceq\mathbf{B}^{-1}$.}:
\begin{equation}
\label{ineq:q_expectation_inv}
\frac{\beta}{\beta+1}\bI_q^{-1}(\btheta_0) \enskip\preceq\enskip \bI_q^{-1}(\hat{\btheta}_{n'}) \enskip\preceq\enskip \frac{\beta}{\beta-1}\bI_q^{-1}(\btheta_0) 
\end{equation}
Now considering the first inequalities of~(\ref{ineq:p_expectation}) and~(\ref{ineq:q_expectation_inv}), multiplying both sides and taking the trace result~(\ref{eq:replacement_inequality}).
\end{proof}
Inequality~(\ref{eq:replacement_inequality}) implies that minimizing $\mbox{\normalfont tr}\left[\bI_q(\hat{\btheta}_{n'})^{-1}\bI_p(\hat{\btheta}_{n'})\right]$ (or an approximation of it) with respect to $q$ in each iteration of FIR-based querying algorithms, namely through the operation $\cala(\call_{n'},\hat{\btheta}_{n'})$ (line~\ref{line:generate_queries} of Algorithm~\ref{alg:AL}), is equivalent to upper bound minimization of the original cost function, i.e. left-hand-side of~(\ref{eq:FIR_variance_upper_bound}).

\subsubsection{Monte-Carlo Approximation}
Computation of Fisher information matrices is intractable unless when the marginal distributions are very simple or when they are restricted to be PMFs over finite number of samples. The latter is widely used in pool-based active learning, when the samples in the pool are assumed to be generated from $p(\bx)$. In such cases, one can simply utilize a Monte-Carlo approximation to compute $\bI_p(\tilde{\btheta})$. More specifically, denote the set of observed instances in the pool by $X_p$. Then the test Fisher information at any $\btheta\in\Omega$ can be approximated by
\begin{equation}
\label{eq:MC_approx_Ip}
\bI_p(\btheta)\approx\hat{\bI}(\btheta;X_p) \enskip := \enskip \frac{1}{|X_p|}\sum_{\bx\in X_p}\sum_{y=1}^c p(y|\bx,\btheta)\nabla_{\btheta}\ell(\btheta;\bx,y)\nabla_{\btheta}^\top\ell(\btheta;\bx,y) \enplus \delta \cdot\mathbb{I}_d
\end{equation}
where $\delta$ is a small positive number and the weighted identity matrix is added to ensure positive definiteness. It is important to remark that when using equation~\ref{eq:MC_approx_Ip}, we are actually utilizing some of the test samples in the training process, hence we cannot use those in $X_p$ in order to evaluate the performance of the trained classifier.

Similarly, $\bI_q(\hat{\btheta}_{n'})$ can be estimated based on a candidate query set $X_q$. Let $X_q$ be the set of samples drawn independently from $q(\bx)$. Then we can have the approximation $\bI_q(\hat{\btheta}_{n'})\approx\hat{\bI}(\hat{\btheta}_{n'};X_q)$. Putting everything together, the best query set $X_q\subseteq X_p$ is chosen to be the one that minimizes the approximate FIR querying objective:
\begin{equation}
\label{eq:hat_Xq_Xp_objective}
\mbox{tr}\left[\hat{\bI}(\hat{\btheta}_{n'};X_q)^{-1}\hat{\bI}(\hat{\btheta}_{n'};X_p)\right].
\end{equation}
Note that this objective is directly written in terms of $X_q$, and therefore the queries can be deterministically determined by fixing all the rest (including the current parameter estimate $\hat{\btheta}_{n'}$) and optimizing with respect to $X_q$. Therefore, such settings are usually called \emph{deterministic} active learning, as opposed to the probabilistic nature of~(\ref{eq:FIRAL}).

\subsubsection{Bound Optimization}
\label{subsubsec:bound_opt}
There are other types of approximation methods occurring in the optimization side. These methods are able to remove part of the unknown variables by doing upper-bound minimization or lower-bound maximization. Recall that in active learning, the querying objective is to be optimized with respect to $q$ (or $X_q$ in pool-based scenario). In a very simple example, when $d=1$, note that the $\bI_p(\btheta_0)$ is a constant scalar in~(\ref{eq:q_optimization}) and hence can be ignored. Hence, in the scalar case, we can simply focus on maximizing the training Fisher information. In the multivariate case, though it is not clear what measure of $\bI_q(\hat{\btheta}_n)$ to optimize, one may choose the objective to be $|\bI_q(\hat{\btheta}_n)|$ (where $|\cdot|$ is the determinant function),\footnote{Similar to \emph{D-optimality} in Optimal Experiment Design~\citep{fedorov1972theory}.} or $\mbox{tr}[\bI_q(\hat{\btheta}_n)]$.\footnote{Similar to \emph{A-optimality} in Optimal Experiment Design~\citep{fedorov1972theory}.} The latter is worth paying more attention due to the following inequality~\citep{yang2000matrix}:
\begin{equation}
\label{eq:trace_inequality}
\mbox{tr}[\bI_q(\btheta_0)^{-1}\bI_p(\btheta_0)] \enskip \leq \enskip \mbox{tr}[\bI_q(\btheta_0)^{-1}]\cdot\mbox{tr}[\bI_p(\btheta_0)].
\end{equation}
Since $\mbox{tr}[\bI_p(\btheta_0)]$ is a constant with respect to $q$, minimizing the right-hand-side of inequality~(\ref{eq:FIR_variance_upper_bound}) can itself be approximated by another upper-bound minimization:
\begin{equation}
\label{eq:upper_bound_minimization}
\argmin_{q}\enskip\mbox{tr}[\bI_q(\btheta_0)^{-1}]
\end{equation}
This helps removing the dependence of the objective to the test distribution $p$. A lower bound can also be established for the FIR. Using the inequality between arithmetic and geometric means of the eigenvalues of $\bI_q(\btheta_0)^{-1}\bI_p(\btheta_0)$, one can see that $d\cdot|\bI_q(\btheta_0)^{-1}|\cdot|\bI_p(\btheta_0)|\leq\mbox{tr}[\bI_q(\btheta_0)^{-1}\bI_p(\btheta_0)]$. Hence, when minimizing the upper-bound by minimizing the trace of $\bI_q(\btheta_0)^{-1}$, one should be careful about the determinant of this matrix as a term influencing the lower-bound of the objective. 

In practice, of course, the minimization in~(\ref{eq:upper_bound_minimization}) can be difficult due to matrix inversion. Thus, sometimes it is further approximated by
\begin{equation}
\label{eq:maximization_trace_Iq}
\argmax_q \enskip\mbox{tr}[\bI_q(\btheta_0)].
\end{equation}
Hence, algorithms that aim to maximize $\mbox{tr}[\bI_q(\btheta_0)]$, indeed introduce three layers of objective approximations through equations~(\ref{eq:trace_inequality}) to~(\ref{eq:maximization_trace_Iq}). As discussed before, the dependence of the objectives in all the layers (in either~(\ref{eq:upper_bound_minimization}) or~(\ref{eq:maximization_trace_Iq})), can be removed by replacing it with the current estimate $\btheta_{n'}$.


\subsection{Some Existing Algorithms}
\label{subsec:algorithms}
In this section, we discuss several existing algorithms for implementing the query selection task based on minimization of FIR. We will analyze these algorithms, sorted according to date of their publication, in the context of our unifying framework.


Besides the categorizations that have already been described in previous sections, it is also useful to divide the querying algorithms into two categories based on the size of $X_q$: \emph{sequential active learning} where a single sample is queried at each iteration, i.e. $|X_q|=1$; and \emph{batch active learning} where the size of the query set is larger than one. The non-singleton query batches are usually generated greedily, with the batch size $|X_q|$ fixed to a constant value.

Table~\ref{tab:alg_summary} lists the algorithms that we reviewed in the following sections together with a summary of their properties and the approximate objective that they optimize for querying. Note that among these algorithms, the one by~\citet{chaudhuri2015convergence} makes extra assumptions as is described in Section~\ref{subsubsec:replacing_theta}.

\begin{table}[t!]
\caption{Reviewed FIR-based active learning algorithms for discriminative classifiers}
\label{tab:alg_summary}
\centering
 \begin{tabular}{|| c | c | c || c c || c c || c c ||} 
 \hline
 \multicolumn{2}{||c|}{\textbf{Algorithm}}  & \textbf{Obj.} & \textbf{Prob.} & \textbf{Det.} & \textbf{Pool} & \textbf{Syn.} & \textbf{Seq.} & \textbf{Batch} \\ [0.5ex] 
 \hline\hline
 1 & \citet{fukumizu2000statistical} & (\ref{eq:maximization_trace_Iq})  & \ding{51} & & & \ding{51} & \ding{51} & \ding{51}  \\
 \hline
 2 & \citet{Zhang2000} & (\ref{eq:maximization_trace_Iq})  &  & \ding{51} & \ding{51} & \ding{51} & \ding{51} &   \\
 \hline
 3 & \small \citet{settles2008analysis} & (\ref{eq:hat_Xq_Xp_objective})  &  & \ding{51} & \ding{51} &  & \ding{51} & \ding{51}  \\
 \hline
 4 & \citet{Hoi2006,hoi2009batch} & (\ref{eq:hat_Xq_Xp_objective})  &  & \ding{51} & \ding{51} &  & \ding{51} & \ding{51}  \\
 \hline
 5 & \citet{chaudhuri2015convergence} & (\ref{eq:replacement_inequality})  & \ding{51} &  & \ding{51} &  & \ding{51} & \ding{51}  \\
 \hline
 \end{tabular}
\end{table}

\vspace{.25cm}
\noindent\textbf{Algorithm 1 } \textit{(\citet{fukumizu2000statistical})}
\vspace{.25cm}

This algorithm is the classification version of the \emph{probabilistic active learning} proposed by~\citet{fukumizu2000statistical} for regression problem. The assumption is that the proposal belongs to a parametric family and is of the form $q(\bx;\balpha)$, where $\balpha$ is the parameter vector of the family. In this \emph{parametric} active learning, the best set of parameters $\hat{\balpha}$ is selected using the current parameter estimate and the query set is sampled from the resulting proposal distribution $X_q\sim q(\bx;\hat{\balpha})$. 

This algorithm makes no use of the test samples and optimizes the simplified objective in~(\ref{eq:maximization_trace_Iq}) to obtain the query distribution $q(\bx)$. Denote the covariates of the current training data set $\call_{n'}$ by $X_{\call}$. As is described in section~(\ref{subsec:practical_issues}), the new trace objective can be approximated by Monte-Carlo formulation using the old queried samples $X_{\call}$ as well as the candidate queries $X_q$ to be selected in this iteration:
\begin{equation}
\label{eq:MC_approx_alg1}
\mbox{tr}\left[\hat{\bI}(\hat{\btheta}_{n'};X_{\call}\cup X_q)\right],
\end{equation}
\begin{algorithm}[!t]
\caption{\citet{fukumizu2000statistical}}
\label{alg:parametric_AL}
\textbf{Inputs:} Current estimation of the parameter $\hat{\btheta}_{n'}$, size of the query set $|X_q|$\\
\textbf{Outputs:} The query set $X_q$\\[-5pt]
\noindent\rule{.5\columnwidth}{.5pt}\\
\SetNlSkip{1em}
\DontPrintSemicolon
\tcc{\footnotesize Parameter Optimization}
\nl$\hat{\balpha} \equals \argmax_{\balpha} \mathbb{E}_{q(\bx;\balpha)}\left[\sum_{y=1}^cp(y|\bx,\hat{\btheta}_{n'})\nabla_{\btheta}^\top\ell(\hat{\btheta}_{n'};\bx,y)\nabla_{\btheta}\ell(\hat{\btheta}_{n'};\bx,y) \right]$\;
\tcc{\footnotesize Sampling from the parametric proposal}
\nl$\bx_i \enskip\sim\enskip q(\bx;\hat{\balpha}) \qquad,i=1,...,|X_q|$\;
\nl\textbf{return} $X_q=\left\{\bx_1,...,\bx_{|X_q|}\right\}$\;
\end{algorithm}

More specifically, the new parameter vector is obtained by maximizing the expected contribution of the queries $X_q$ generated from $q(\bx;\balpha)$ to this objective. Taking expectation of~(\ref{eq:MC_approx_alg1}) with respect to $q(\bx;\balpha)$ yields:
\begin{equation}
\label{eq:parametric_objective}
    \mathbb{E}_{q(\bx;\balpha)}\left[\mbox{tr}\left[\hat{\bI}(\hat{\btheta}_{n'};X_{\call}\cup X_q)\right]\right] \equals 
    \mbox{tr}\left[\frac{n'}{n' + |X_q|} \hat{\bI}(\hat{\btheta}_{n'};X_{\call}) + \frac{1}{n'+|X_q|}\mathbb{E}_{q(\bx;\balpha)}\left[\hat{\bI}(\hat{\btheta}_{n'};X_q)\right] \right].
\end{equation}
Recall that $n'$ is the size of the current training data set $\call_{n'}$. The first term in~(\ref{eq:parametric_objective}) is independent of the query set $X_q$ (assuming that the size $|X_q|$ is fixed to a constant), hence we focus only on the second term in our optimization. Noting that the queries are generated independently, we can rewrite this term as:
\begin{align}
\label{eq:expected_contribution}
    \mathbb{E}_{q(\bx;\balpha)}\bigg[\mbox{tr}[\hat{\bI}(\hat{\btheta}_{n'};X_q)]\bigg] &\equals 
    \mathbb{E}_{q(\bx;\balpha)}\left[\frac{1}{|X_q|}\sum_{\bx\in X_q}\mbox{tr}\left[\hat{\bI}(\hat{\btheta}_{n'};\bx)\right]\right] \enskip -\enskip (|X_q|-1)\delta\cdot\mathbb{I}_d  \nonumber\\
    &\equals \mathbb{E}_{q(\bx;\balpha)}\bigg[\mbox{tr}\left[\hat{\bI}(\hat{\btheta}_{n'};\bx)\right]\bigg] \enskip -\enskip (|X_q|-1)\delta\cdot\mathbb{I}_d
\end{align}
From equation~(\ref{eq:expected_contribution}), if we are to select a single query, the parameter vector $\balpha$ can be obtained by maximizing the expected contribution of that single query to the trace objective, that is:
\begin{align}
\label{eq:alphan_optimization}
    \hat{\balpha} &\equals \argmax_{\balpha} \enskip \mathbb{E}_{q(\bx;\balpha)}\bigg[\mbox{tr}\left[\hat{\bI}(\hat{\btheta}_{n'};\bx)\right]\bigg] \nonumber\\
    &\equals \argmax_{\balpha} \enskip \mathbb{E}_{q(\bx;\balpha)}\left[\sum_{y=1}^cp(y|\bx,\hat{\btheta}_{n'})\nabla_{\btheta}^\top\ell(\hat{\btheta}_{n'};\bx,y)\nabla_{\btheta}\ell(\hat{\btheta}_{n'};\bx,y) \right] 
\end{align}
The optimization~(\ref{eq:alphan_optimization}) does not depend on $X_{\call}$, and therefore we do not need to explicitly feed this algorithm with $\call$. All it needs is an estimation of the parameter $\hat{\btheta}_{n'}$. The two-step procedure of generating queries from parametric query distribution is shown in Algorithm~\ref{alg:parametric_AL}. This algorithm can be used in both sequential and batch modes by changing the number of samples drawn from $q(\bx;{\alpha})$.

We emphasize that Algorithm~\ref{alg:parametric_AL} is probabilistic, meaning that with any fixed parameter estimate $\hat{\btheta}_{n'}$, the next set of queries are \emph{not} deterministically selected. The optimization is performed with respect to the parameters of the proposal distribution, which are then used to sample $X_q$. \citet{fukumizu2000statistical} claims that introducing such randomness into active learning, which increases exploration against exploitation, may prevent the algorithm from falling into local optima. Also note that this algorithm is not pool-based, meaning that it does not select the queries from a pool of observed instances, although could be constrained to do so.

\vspace{.25cm}
\noindent\textbf{Algorithm 2} \textit{(\citet{Zhang2000})}
\vspace{.25cm}

\citet{Zhang2000} started from optimization problem~(\ref{eq:maximization_trace_Iq}), and introduced even additional simplifications to it, specifically considering the use of a binary logistic regression classifier. Here, we discuss their formulation using a general discriminate framework. 

In their algorithm, a single query is selected at each iteration. Denote it by $X_q = \{\bx_q\}$. 
Similar to the previous section, the Fisher information matrix $\bI_q$ can be approximated by Monte-Carlo approximation. \citet{Zhang2000} discarded the expectation with respect to the proposal distribution in~(\ref{eq:alphan_optimization}) or equivalently consider $q$ to be a uniform distribution. Therefore, the optimization with respect to parameters turned into a direct optimization with respect to the single query $\bx_q$:
\begin{align}
\label{eq:trace_Iq_expanded}
    \bx_q \equals \argmax_{\bx\in X} \enskip \sum_{y=1}^cp(y|\bx,\hat{\btheta}_{n'})\nabla_{\btheta}^\top\ell(\hat{\btheta}_{n'};\bx,y)\nabla_{\btheta}\ell(\hat{\btheta}_{n'};\bx,y)
\end{align}
This single-step deterministic approach, shown in Algorithm~\ref{alg:deterministic_AL_Zhang}, is very similar to the probabilistic approach described above, except that there is no intermediate parameter optimization step.

\begin{algorithm}[!t]
\caption{\citet{Zhang2000}}
\label{alg:deterministic_AL_Zhang}
\textbf{Inputs:} Current estimation of the parameter $\hat{\btheta}_{n'}$\\
\textbf{Outputs:}  The query singleton set $X_q=\{\bx_q\}$\\[-5pt]
\noindent\rule{.5\columnwidth}{.5pt}\\
\SetNlSkip{1em}
\DontPrintSemicolon
\nl$\bx_q \enskip\gets\enskip \argmax_{\bx} \enskip \sum_{y=1}^cp(y|\bx,\hat{\btheta}_{n'})\nabla_{\btheta}^\top\ell(\hat{\btheta}_{n'};\bx,y)\nabla_{\btheta}\ell(\hat{\btheta}_{n'};\bx,y)$\;
\nl\textbf{return} $X_q=\{\bx_q\}$\;
\end{algorithm}

It is important to note that Algorithm~\ref{alg:deterministic_AL_Zhang} can be used in pool-based active learning as well. This can be done by constraining $\bx_q$ to be a member of a pool of samples, in which case it can even be extended to batch querying by sorting the unlabeled samples based on their objective values and taking the highest ones. However, such iterative optimization is not efficient, because the resulting queries will most probably be close to each other and therefore contain redundant information.

\newpage
\vspace{.25cm}
\noindent\textbf{Algorithm 3} \textit{(\citet{settles2008analysis})}
\vspace{.25cm}

Inspired by \citet{Zhang2000}, \citet{settles2008analysis} employed Fisher information ratio to develop a pool-based active learning, which can be used in either sequential or batch querying. The pool that is used here is the set of unlabeled samples, $X_p$, which are assumed to be drawn from the test marginal $p(\bx)$. Queries are chosen from $X_p$, that is $X_q\subseteq X_p$.  The test Fisher information matrix can be approximated by Monte-Carlo simulation over the samples in $X_p$, meaning $\hat{\bI}\left(\hat{\btheta}_{n'};X_p\right)$. 
Similar to Algorithm~\ref{alg:parametric_AL}, the updated training Fisher information matrix after querying a set $X_q$ can be approximated by $\hat{\bI}\left(\hat{\btheta}_{n'};X_{\call}\cup X_q \right)$. Thus, since we do have an approximation of both Fisher information matrices, the objective to minimize is chosen to be in the form of~(\ref{eq:hat_Xq_Xp_objective}).

Similar to the \citet{Zhang2000} algorithm, the proposal distribution $q$ is ignored in the objective (or equivalently considered as being uniform). An additional assumption \citet{settles2008analysis} made to simplify the optimization task is:
\begin{align}
\label{eq:settles_simplifying_assumption}
\argmin_{X_q\subset X_p} \enskip \mbox{tr}&\left[\hat{\bI}\left(\hat{\btheta}_{n'};X_{\call}\cup X_q\right)^{-1}\hat{\bI}\left(\hat{\btheta}_{n'};X_p\right)\right] \enskip\approx\enskip
\argmin_{X_q\subset X_p} \enskip \mbox{tr}\left[\hat{\bI}\left(\hat{\btheta}_{n'};X_q\right)^{-1}\hat{\bI}\left(\hat{\btheta}_{n'};X_p\right)\right].
\end{align}
This simplified optimization is easy to implement for sequential active learning. However, the combinatorial optimization required for batch active learning can easily become intractable. As shown in Algorithm~\ref{alg:deterministic_AL_Settles}, \citet{settles2008analysis} used a greedy approach to do this optimization (the inner loop).

\begin{algorithm}[!t]
\caption{\citet{settles2008analysis}}
\label{alg:deterministic_AL_Settles}
\textbf{Inputs:} Current estimation of the parameter $\hat{\btheta}_{n'}$, the set of unlabeled samples $X_p$, , size of the query set $|X_q|$\\
\textbf{Outputs:} The query set $X_q$\\[-5pt]
\noindent\rule{.5\columnwidth}{.5pt}\\
\SetNlSkip{1em}
\DontPrintSemicolon
\tcc{\footnotesize Initializing the query set for this iteration}
\nl$X_q \enskip \gets \enskip \varnothing$\;
\tcc{\footnotesize  The loop for greedy batch querying}
\nl\For{$j=1\to |X_q|$}{
\tcc{\footnotesize Query optimization and adding the result into the query set}
\nl$X_q \enskip\gets\enskip X_q \enskip \cup \enskip \argmin_{\bx\in X_p} \mbox{tr}\left[\hat{\bI}\left(\hat{\btheta}_{n'};\bx \right)^{-1}\hat{\bI}\left(\hat{\btheta}_{n'};X_p\right)\right]$\label{line:settles_optimization}\;
\tcc{\footnotesize Removing the selected queries from the pool}
\nl$X_p \enskip \gets \enskip X_p - X_q$\;
}
\nl\textbf{return} $X_q$\;
\end{algorithm}

\vspace{.25cm}
\noindent \textbf{Algorithm 4} \textit{(\citet{Hoi2006} and~\citet{hoi2009batch})}
\vspace{.25cm}

The algorithms proposed by \citet{Hoi2006} and \citet{hoi2009batch} are very similar to the one developed by \citet{settles2008analysis}, described above, except that they use a more sophisticated optimization method. Their method shown in Algorithm~\ref{alg:deterministic_AL_Hoi}, is different from Algorithm~\ref{alg:deterministic_AL_Settles} mainly in the way that it greedily chooses the query at each inner loop iteration of the algorithm. While Algorithm~\ref{alg:deterministic_AL_Settles} exclusively considers the contribution of each $\bx\in X_q$, ignoring the samples selected in the previous iterations (hence $\hat{\bI}\left(\hat{\btheta}_{n'};\bx \right)$ in line~\ref{line:settles_optimization} of Algorithm~\ref{alg:deterministic_AL_Settles}), Algorithm~\ref{alg:deterministic_AL_Hoi} takes into account all the queries chosen so far (hence $\hat{\bI}\left(\hat{\btheta}_{n'};X_q\cup\{\bx\} \right)$ in line~\ref{line:hoi_optimization} in Algorithm~\ref{alg:deterministic_AL_Hoi}).

\citet{Hoi2006} and \citet{hoi2009batch} showed that when using \emph{binary logistic regression} classifier, their optimization~(\ref{eq:settles_simplifying_assumption}) can be done by maximizing a \emph{submodular} set function with respect to the query set $X_q$. This allowed them to use the well-known iterative algorithm proposed by~\citet{Nemhauser1978a}, which guarantees a tight lower-bound for maximization of submodular and monotone set functions.

In the rest of this section, we show that minimizing this objective obtained from the above-mentioned assumptions, can be efficiently approximated by a monotonically submodular maximizing under \emph{any} discriminative classifier. This is a generalization of the result derived by~\citet{Hoi2006} that is obtained in case of using logistic regression classifier. As a consequence, FIR can be efficiently optimized with guaranteed tight bounds~\citep{Nemhauser1978a,Nemhauser1978b}.
The following lemma shows that~(\ref{eq:settles_simplifying_assumption}) is approximately equivalent to maximizing a simplified set function, for any unlabeled sample pool $X_p$:

\begin{algorithm}[!t]
\caption{\citet{Hoi2006, hoi2009batch}}
\label{alg:deterministic_AL_Hoi}
\textbf{Inputs:} Current estimation of the parameter $\hat{\btheta}_{n'}$, the set of unlabeled samples $X_p$, size of the query set $|X_q|$\\
\textbf{Outputs:} The query set $X_q$\\[-5pt]
\noindent\rule{.5\columnwidth}{.5pt}\\
\SetNlSkip{1em}
\DontPrintSemicolon
\tcc{\footnotesize Initializing the query set}
\nl$X_q\gets \varnothing$\;
\tcc{\footnotesize  The loop for greedy batch querying}
\nl\For{$j=1\to |X_q|$}{\label{line:Hoi_greedy_max_start}
\tcc{\footnotesize Query optimization}
\nl$\tilde{\bx} \equals \argmin_{\bx\in X_p} \mbox{tr}\left[\hat{\bI}\left(\hat{\btheta}_{n'};X_q\cup\{\bx\} \right)^{-1}\hat{\bI}\left(\hat{\btheta}_{n'};X_p\right)\right]$\label{line:hoi_optimization}\;
\tcc{\footnotesize Add the selected query into the query set}
\nl$X_q\gets X_q\cup\{\tilde{\bx}\}$\;
\tcc{\footnotesize Remove the selected instance from the pool}
\nl$X_p\gets X_p-\{\tilde{\bx}\}$\label{line:Hoi_greedy_max_finish}\;
}
\nl\textbf{return} $X_q$\;
\end{algorithm}

\begin{lemma}
\label{lemma:equivalent_set_optimization}
Let $X_p,X_q\subseteq X$ be two non-empty and finite subsets of samples randomly generated from $p(\bx)$ and its resample distribution $q(\bx)$, respectively, such that $X_q\subset X_p$, and the parameter $\delta\geq0$ in~(\ref{eq:MC_approx_Ip}) is a small constant. If assumptions~\ref{item:parameter_dependence}, \ref{item:iid}, \ref{item:differentiability} and \ref{item:FisherInformation} hold, then the following optimization problems are approximately equivalent for some function $g_{\btheta}:X\times\{1,...,c\}\times X\to\mathbb{R}^+$, $d$-dimensional non-zero vector $\bv_{\btheta}$ depending on $\bx$ and $y$, and for all $\btheta\in\Omega$ :
\begin{subequations}
\begin{align}
\label{eq:original_set_optimization}
(i)\quad \argmin_{X_q\subset X_p} \enskip &\mbox{\normalfont{tr}}\left[\hat{\bI}(\btheta;X_q)^{-1}\hat{\bI}(\btheta;X_p)\right] \\
\label{eq:equivalent_set_optimization}
(ii)\quad \argmax_{X_q\subset X_p} \enskip &\sum_{\bx\in X_p-X_q}\sum_{y=1}^c\frac{-1}{\delta\cdot\|\bv_{\btheta}(\bx,y)\|^{-2} + \sum_{\bx'\in X_q} g_{\btheta}(\bx,y,\bx')}
\end{align}
\label{eq:set_optimizations}
\end{subequations}
The approximation is more accurate for smaller $\delta$ and well-conditioned Monte-Carlo approximation of proposal Fisher information matrix.
\end{lemma}
The proof can be found in Appendix~\ref{app:proof_Lemma}. 
Note that Lemma~\ref{lemma:equivalent_set_optimization}, as stated above, does not depend on the size of $X_q$. However, just as before, in practice it is usually assumed that $|X_q|>0$ is fixed and therefore the optimizations in~(\ref{eq:set_optimizations}) should be considered with cardinality constraint. In general, combinatorial maximization problems can turn out to be intractable. Next, it is shown that the objective at hand is a monotonically submodular set function in terms of $X_q$ and therefore can be maximized efficiently with a greedy approach such as that shown in Algorithm~\ref{alg:deterministic_AL_Hoi}.

\begin{theorem}
\label{thm:submodularity}
Suppose $f_{\btheta}:2^{X_p}\to\mathbb{R}$ is defined as:
\begin{equation}
\label{eq:submodular_objective}
f_{\btheta}(X_q) \equals \sum_{\bx\in X_p-X_q}\sum_{y=1}^c\frac{-1}{\delta\cdot\|\bv_{\btheta}(\bx,y)\|^{-2} + \sum_{\bx'\in X_q} g_{\btheta}(\bx,y,\bx')},\quad \forall X_q\subseteq X_p
\end{equation}
with $\bv_{\btheta}$ a $d$-dimensional vector depending on $\bx$ and $y$, and $g_{\btheta}$ defined in~(\ref{eq:defn_g}).
Then $f_{\btheta}$ is a submodular and monotone (non-decreasing) set function for all $\btheta\in\Omega$.
\end{theorem}
The proof is in Appendix~\ref{app:proof_submodularity}.
The result above, together with Lemma~\ref{lemma:equivalent_set_optimization}, imply that the objective of~(\ref{eq:equivalent_set_optimization}) is a monotonically increasing set function with respect to $X_q$. 
Below we present the main result that guarantees tight bounds for greedy maximization of monotonic submodular set functions. Details of this result, which is also shown to be the optimally efficient solution to submodular maximization, can be found in the seminal papers by~\citet{Nemhauser1978a} and~\citet{Nemhauser1978b}.
\begin{theorem}[\citet{Nemhauser1978a}]
\label{thm:greedy_bound}
Let $f_{\btheta}:2^{X_p}\to\mathbb{R}$ be any submodular and nondecreasing set function with $f(\varnothing)=0$ \footnote{This can always be assumed since maximizing a general set function $f(X_q)$ is equivalent to maximizing its adjusted version $g(X_q):=f(X_q) - f(\varnothing)$, which satisfies $g(\varnothing)=0$.}. If $X_q$ is the output of a greedy maximization algorithm, and $X_q^*$ is the optimal maximizer of $f_{\btheta}$ with a cardinality constraint (fixed $|X_q|$), then we have:
\begin{align}
\label{eq:deterministic_greedy_bound}
f_{\btheta}(X_q)\enskip \geq \enskip\left[1-\left(\frac{|X_q|-1}{|X_q|}\right)^{|X_q|} \right]f_{\btheta}(X_q^*)\enskip\geq\enskip \left(1-\frac{1}{e}\right)f_{\btheta}(X_q^*).
\end{align}
\end{theorem}
In Algorithm~\ref{alg:deterministic_AL_Hoi}, the inner loop (lines~\ref{line:Hoi_greedy_max_start} to~\ref{line:Hoi_greedy_max_finish}) implements the minimization in~(\ref{eq:settles_simplifying_assumption}) greedily. We have seen above that this set minimization is approximately equivalent to maximizing a submodular and monotone set maximization, which, in turn, is shown to be efficient.

\newpage
\vspace{.25cm}
\noindent \textbf{Algorithm 5} \textit{(\citet{chaudhuri2015convergence})}
\vspace{.25cm}

This algorithm uses FIR for doing a probabilistic pool-based active learning. It has extra assumptions in comparison to our general framework, which are briefly explained  in Section~\ref{subsubsec:replacing_theta}. Note that these assumptions are to be made as well as those listed in Section~\ref{sec:framework_assumptions}. In such settings, \citet{chaudhuri2015convergence} gave a finite-sample theoretical analysis for FIR when applied to pool-based active learning. 

More specifically, suppose $p(\bx)$ is a uniform PMF and $q(\bx)$ is a general PMF, both defined over the pool $X_p$. Using the notations in~(\ref{eq:Chaudhuri_assumptions}), the training Fisher information can be written as $\bI_q(\hat{\btheta}_{n'})=\sum_{\bx\in X_p}q(\bx)\bI(\hat{\btheta}_{n'},\bx)$. Now, assuming that $\bI_p(\hat{\btheta}_{n'})$ has a singular decomposition of the form $\sum_{j=1}^d\sigma_j\mathbf{u}_j\mathbf{u}_j^\top$, FIR can be written as:
\begin{align}
\label{eq:chaudhuri_objective}
\mbox{tr}\bigg[\bI_q(\hat{\btheta}_{n'})^{-1}\bI_p(\hat{\btheta}_{n'})\bigg] &\equals \sum_{j=1}^d\sigma_j\mbox{tr}\bigg[\bI_q(\hat{\btheta}_{n'})^{-1}\mathbf{u}_j\mathbf{u}_j^\top\bigg] \nonumber\\
&\equals \sum_{j=1}^d\sigma_j\mathbf{u}_j^\top\bI_q(\hat{\btheta}_{n'})^{-1}\mathbf{u}_j
\end{align}
Minimizing the last term in~(\ref{eq:chaudhuri_objective}) with respect to PMF $\{q(\bx) | \bx\in X_p\}$ is equivalent to a semidefinite programming after introducing a set of auxiliary variables $t_j, j=,1...,d$ and applying Schur complements~\citep{vandenberghe1996semidefinite}:
\begin{align}
\label{eq:chaudhuri_SDP}
\argmin_{q(\bx), \bx\in X_p} \quad &\sum_{j=1}^d \sigma_j t_j \\
\mbox{such that } \enskip 
&\begin{bmatrix} 
t_j & \mathbf{u}_j^\top \\
\mathbf{u}_j & \sum_{\bx\in X_p} q(\bx)\bI(\hat{\btheta}_{n'},\bx)
\end{bmatrix}
\succeq 0, \nonumber\\
& \sum_{\bx\in X_p}q(\bx)=1. \nonumber
\end{align}
The steps for this querying method is shown in Algorithm~\ref{alg:probabilistic_chaudhuri}. Note that the solution to~(\ref{eq:chaudhuri_SDP}) is slightly modified by mixing it with the uniform distribution over the pool. Such modification is mainly to establish their theoretical derivations. The mixing coefficient, $0\leq\lambda\leq1$ reciprocally depends on the number of queries. More specifically, \citet{chaudhuri2015convergence} made it equal to $1-\frac{1}{|X_q|^{1/6}}$. That is, as the number of queries increases, $\lambda$ shrinks and so does the modification. Furthermore, in their analysis, they assumed that sampling from $\tilde{q}(\bx)$ (line~\ref{line:sampline_qtilde} of Algorithm~\ref{alg:probabilistic_chaudhuri}) is done \emph{with replacement}. That is, label of a given sample might be queried multiple times.

\begin{algorithm}[!t]
\caption{\citet{chaudhuri2015convergence}}
\label{alg:probabilistic_chaudhuri}
\textbf{Inputs:} Current estimation of the parameter $\hat{\btheta}_{n'}$, the set of unlabeled samples $X_p$, size of the query set $|X_q|$\\
\textbf{Outputs:} The query set $X_q$\\[-5pt]
\noindent\rule{.5\columnwidth}{.5pt}\\
\SetNlSkip{1em}
\DontPrintSemicolon
\tcc{\footnotesize  Solving the semidefinite programming}
\nl $q(\bx) \enskip\gets \enskip$ solution to~(\ref{eq:chaudhuri_SDP})\;
\tcc{\footnotesize  Modification of the solution}
\nl $\tilde{q}(\bx) \enskip\gets \enskip \lambda q(\bx) + (1-\lambda)U(\bx)$ \;
\tcc{\footnotesize  Sampling with replacement from the modified proposal}
\nl$\bx_i \enskip\sim\enskip \tilde{q}(\bx) \qquad, i=1,...,|X_q|$\label{line:sampline_qtilde}\;
\nl \textbf{return} $X_q=\left\{\bx_1,...\bx_{|X_q|}\right\}$\;
\end{algorithm}

\subsubsection{Comparison with Other Information-theoretic Objectives}
\label{subsubsec:comparison}
In the last part of this section, we compare FIR and two other common querying objectives from the field of information theory. Entropy of class labels and mutual information between labeled and unlabeled samples are two other common active learning objectives. Their goal is mainly to get the largest possible amount of information about \emph{class labels of unlabeled samples} from each querying iteration, hence naturally pool-based. 

Entropy-based querying, also known as uncertainty sampling, directly measures the uncertainty with respect to class label of each unlabeled sample and query those with highest uncertainty. It has been hugely popular due to its simplicity and effectiveness especially in sequential active learning. However, it does not consider interaction between samples when selecting multiple queries, which can cause querying very similar samples (redundancy). Therefore, uncertainty sampling shows relatively poor performance in batch active learning. Mutual information, on the other hand, does not suffer from redundancy, however, it requires a much higher computational complexity.

These two objectives directly measure the amount of information each batch can have with respect to the class labels (hence prediction-based), as opposed to Fisher information as a measure of information regarding the distribution parameters (hence inference-based). However, there is no guarantee that by minimizing uncertainty of the class labels (or equivalently, choosing queries with highest amount of information about class labels), the prediction accuracy also increases. Whereas, as we showed earlier, FIR upper-bounds the expected asymptotic variance of a parameter inference loss function. From this point of view, FIR has a closer relationship with the performance of a classifier.

\setlength{\extrarowheight}{5pt}
\begin{table}[t!]
\caption{Computational complexity of different querying algorithms}
\label{tab:complexity}
\centering
 \begin{tabular}{| c | c |} 
 \hline
 \textbf{Algorithm} & \textbf{Complexity}  \\ [0.5ex] 
 \hline\hline
 Entropy & $O(|X_p|cd)$ \\[0.5ex]
 \hline
 Mutual Information & $O\left(|X_p|\cdot|X_q|\cdot c^{|X_q|+1}d\right)$ \\[0.5ex]
 \hline\hline
 \citet{Zhang2000} &   $O(|X_p| cd)$ \\[0.5ex]
 \hline
 \small \citet{settles2008analysis} & $O\left(|X_q|\cdot|X_p|\cdot(cd+d^3)\right)$ \\[0.5ex]
 \hline
 \citet{Hoi2006,hoi2009batch} &  $O\left(|X_q|\cdot|X_p|\cdot(cd+cd|X_q|+d^3)\right)$ \\[0.5ex]
 \hline
 \citet{chaudhuri2015convergence} &   $O\left( d^3|X_p|^2 + d^4|X_p| + d^5\right)$ \\[0.5ex]
 \hline
 \end{tabular}
\end{table}

Table~\ref{tab:complexity} shows computational complexity of the querying objectives. The algorithm by \citet{fukumizu2000statistical} is excluded from this table since it cannot be used in pool-based sampling. Also note that the complexity reported for mutual information is for the case when it is optimized greedily. Nevertheless, it still contains an exponential term in its complexity. Entropy-based and \citet{Zhang2000} have the lowest complexity, but in the expense of introducing redundancy into the batch of queries.  Algorithms by \citet{settles2008analysis}, \citet{Hoi2006,hoi2009batch} and \citet{chaudhuri2015convergence} become very expensive when $d$ is large, whereas mutual information can easily get intractable for selecting batches of higher size (large $|X_q|$). Observe that algorithm by \citet{Hoi2006, hoi2009batch} is more expensive than \citet{settles2008analysis}. Recall that despite similarities in appearance, the former guarantees tight bound for its greedy optimization, whereas the latter does not. 

The complexity for the algorithm by \citet{chaudhuri2015convergence} is computed assuming that a barrier method (following path) is used as its numerical optimization~\citep{boyd2004convex}. From Table~\ref{tab:complexity}, this algorithm is the only one whose complexity increases quadratically with size of the pool $|X_p|$, and therefore can get significantly slow for huge pools. Furthermore, it does not depend on $|X_q|$ since the optimization in~(\ref{eq:chaudhuri_SDP}) as its main source of computation, only depends on $|X_p|$ and $d$ (computing $\bI(\hat{\btheta}_{n'},\bx)$ is assumed to cost $O(1)$ for each $\bx\in X_p$ as it is taken to be independent of $y$).

%% file: Appendices.tex
\section{Statistical Background}
\label{sec:statistical_background}
Asymptotic analysis plays an important role in statistics. It considers the extreme cases where the number of observations is increased with no bounds. In such scenarios, discussions on different notions of convergence of the sequence of random variables naturally arise. Generally speaking, there are three major types of stochastic convergence: \emph{convergence in probability, convergence in law (distribution)} and \emph{convergence with high probability (almost surely)}. 
Here, we focus on the two former modes of convergence, discuss two fundamental results based on them and formalize our notations regarding parameter estimators. Further details of the following definitions and results can be found in any standard statistical textbook such as~\citet{lehmann1998theory}. 
\subsection{Convergence of Random Variables}
Throughout this section, $\{\bz_1,\bz_2,...,\bz_n,...\}$, denoted simply by $\{\btheta_n\}$, is a sequence of  multivariate random variables lying in $\Omega\subseteq\mathbb{R}^d$. Also suppose that $\bz_0$ is a constant vector and $\tilde{\bz}$ is another random variable in the same space $\Omega$.

\begin{definition}
\label{defn:convergence_in_prob}
We say that the sequence $\{\bz_n\}$ \emph{converges in probability} to $\bz_0$ and write $\bz_n\overset{P}{\rightarrow}\bz_0$, iff for every $\varepsilon>0$ we have:
\begin{equation}
P(|\theta_{ni}-\theta_{0i}|>\varepsilon) \enskip \to \enskip 0,\quad \mbox{for all } i=1,...,d.
\end{equation}
\end{definition}
Convergence in probability is invariant with respect to any continuous mapping:
\begin{proposition}[\citet{brockwell1991time}, Proposition 6.1.4]
\label{prop:convergent_continuous_mappings}
If $\bz_n\overset{P}{\to}\bz_0$ and $g:\Omega\to\mathbb{R}$ is a continuous function at $\bz=\bz_0$, then $g(\bz_n)\overset{P}{\to}g(\bz_0)$.
\end{proposition}

\begin{definition}
\label{defn:convergence_in_law}
We say that a sequence $\{\bz_n\}$ \emph{converges in law (in distribution)} to the random variable $\tilde{\bz}$ and write $\bz_n\overset{L}{\rightarrow}\tilde{\bz}$, iff the sequence of their joint CDFs, $F_n$, point-wise converges to the joint CDF of $\tilde{\bz}$:
\begin{equation}
\label{eq:convergence_in_law}
F_n(\mathbf{a}) = P(\theta_{n1}\leq a_1,...,\theta_{nd}\leq a_d) \enskip \rightarrow \enskip F(\mathbf{a}) = P(\tilde{\theta}_{1}\leq a_1,...,\tilde{\theta}_{d}\leq a_d) \quad \forall\mathbf{a}\in\calc_{F}\subseteq \mathbb{R}^d,
\end{equation}
where $\calc_{F}$ is the set of continuity points of the CDF $F$.
\end{definition}
Equation~(\ref{eq:convergence_in_law}) means that for large values of $n$, the distribution of $\bz_n$ can be well approximated by the distribution of $\tilde{\bz}$. Note that throughout this paper, for simplicity, we say that a random sequence $\{\bz_n\}$ converges to a distribution with density function $p(\bz)$, or write $\bz_n\overset{L}{\to} p(\bz)$, instead of fully saying that $\{\bz_n\}$ converges in law to a random variable with that distribution. 

Note that $\bz_n \overset{P}{\to}\bz_0$ suggests that $\bz_n-\bz_0 \overset{L}{\to}\delta(\bz)$ where $\delta$ is the Kronecker delta function, which can be viewed as the density function of a degenerate distribution at $\bz=\mathbf{0}$. This, however, does not give any information about the speed with which $\bz_n$ converges to $\bz_0$. In order to take the speed into account, we consider the convergent distribution of the sequence $\{a_n\cdot(\bz_n - \bz_0)\}$, where $a_n$ is any sequence of positive integers and $a_n\to\infty (n\to\infty)$. In practice $a_n$ is usually considered to have the form $n^r$ with $r>0$.

\begin{definition}
\label{defn:convergence_rate}
Assume $\bz_n\overset{P}{\to}\bz_0$. We say that the sequence $\{\bz_n\}$ converges to $\bz_0$ with \emph{rate of convergence} $r>0$, iff $n^r(\bz_n-\bz_0)$ converges in law to a random variable with non-degenerate distribution. Furthermore, the non-degenerate distribution is the \emph{asymptotic distribution} of $\bz_n$.
\end{definition}

\noindent Next, we discuss some of the classic results in asymptotic statistics:
\begin{theorem}[\bfseries Law of Large Numbers, \citet{brockwell1991time}]
\label{thm:LLN}
Let $\btheta_1,...,\btheta_n$ be a set of independent and identically distributed (i.i.d) samples. If $\mathbb{E}[\btheta_i] = \bmu$, then
\begin{equation}
\label{eq:LLN}
\bar{\btheta}_n \equals \frac{1}{n}\sum_{i=1}^n\btheta_i \overset{P}{\rightarrow}\bmu.
\end{equation}
\end{theorem}

\begin{theorem}[\bfseries Central Limit Theorem, \citet{lehmann1998theory}]
\label{thm:CLT}
Let $\btheta_1,...,\btheta_n$ be a set of i.i.d samples with mean $\mathbb{E}[\btheta_i]=\bmu$ and covariance $\mbox{\normalfont Cov}[\btheta_i] =\bSigma$ (with a symmetric and positive semi-definite matrix $\bSigma$), then the sequence of sample averages $\left\{\bar{\btheta}_n\right\}$ with $\bar{\bz}_n=\frac{1}{n}\sum_{i=1}^n\btheta_i$ converges to the true mean with convergence rate $1/2$. Moreover, its asymptotic distribution is a zero-mean Gaussian distribution with covariance matrix $\bSigma$, that is:
\begin{equation}
\label{eq:CLT}
\sqrt{n}\cdot(\bar{\btheta}_n - \bmu) \enskip\overset{L}{\rightarrow}\enskip \caln(\mathbf{0},\bSigma).
\end{equation}
\end{theorem}

\noindent The following results are very useful when deriving the asymptotic distribution of a random sequence under a continuous mapping:
\begin{theorem}
\label{thm:multivariate_delta_method}
{\normalfont\bfseries (Multivariate Delta Method, first order, \citet{lehmann1998theory})} Let $\{\bz_n\}$ be a sequence of random variables such that it converges to $\bz_0$ with rate of convergence $1/2$ and a normal asymptotic distribution, that is $\sqrt{n}\cdot(\bz_n-\bz_0)\overset{L}{\to}\caln(\mathbf{0},\bSigma)$. If $g:\mathbb{R}^d\to\mathbb{R}$ is a continuously differentiable mapping and $\nabla_{\bz}g(\bz_0)\neq\mathbf{0}$, then
\begin{equation}
\sqrt{n}\cdot \bigg[g(\bz_n) - g(\bz_0)\bigg]\enskip \overset{L}{\to} \enskip \caln\left(0,\nabla_{\bz}^\top g(\bz_0)\bSigma\nabla_{\bz} g(\bz_0)\right).
\end{equation}
\end{theorem}

\begin{theorem}[\bfseries Multivariate Delta Method, second order]
\label{thm:delta_method_zeroder}
Let $\{\bz_n\}$ be a sequence of random variables such that it converges to $\bz_0$ with rate of convergence $1/2$ and a normal asymptotic distribution, that is $\sqrt{n}\cdot(\bz_n-\bz_0)\overset{L}{\to}\caln(\mathbf{0},\bSigma)$. If $g:\mathbb{R}^d\to\mathbb{R}$ is a continuously differentiable mapping where $\nabla_{\bz}g(\bz_0)=\mathbf{0}$ and $\nabla_{\bz}^2g(\bz_0)$ is non-singular in a neighborhood of $\bz_0$, then the sequence $\{g(\bz_n) - g(\bz_0)\}$ converges in law to a mixture of random variables with first-degree Chi-square distributions, and the rate of convergence is one. More specifically,
\begin{equation}
\label{eq:delta_method_zeroder}
n\cdot \bigg[g(\bz_n) - g(\bz_0)\bigg]\enskip \overset{L}{\to} \enskip \sum_{i=1}^d\lambda_{i}\chi_1^2,
\end{equation}
where $\lambda_{i}$'s are eigenvalues of $\bSigma^{1/2}\nabla_{\bz}g(\bz_0)\bSigma^{1/2}$.
Moreover, variance of this asymptotic distribution can be written as
\begin{equation}
\label{eq:var_delta_method_zeroder}
\frac{1}{2}\left\|\bSigma^{1/2}\nabla_{\bx}^2g(\bx_0)\bSigma^{1/2} \right\|_F^2,
\end{equation}
where $\|\cdot\|_F$ is the Frobenius norm.
\end{theorem}
\begin{proof}
For proof see Appendix~\ref{app:proof_2nd_Delta_Method}.
\end{proof}

\subsection{Parameter Estimation}
\label{subsec:parameter_estimation}
\noindent Now suppose that the set of independent and identically distributed (i.i.d) set of samples $\bx_1,...\bx_n$ are generated from an underlying distribution that belongs to a parametric family, for which the density function $p(\bx|\btheta)$ can be represented by a multivariate parameter vector $\btheta$. Assume the true parameter is $\btheta_0$, that is $\{\bx_i\}\sim p(\bx|\btheta_0),i=1,...,n$. An \emph{estimator} $\btheta_n=\btheta(\bx_1,...,\bx_n)$ is a function that maps the observed random variables to a point in the parameter space $\Omega$. The subscript $n$ in $\btheta_n$ indicates its dependence on the sample size. Since the observations are generated randomly, the estimators are also random and thus $\{\btheta_n\}$ can be viewed as a sequence of random variables. There are some reserved terms for such a sequence, which we introduce in the remaining of this section:

\begin{definition}[\bfseries Consistency]
We say that an estimator $\btheta_n$ is \emph{consistent} iff $\btheta_n\overset{P}{\to}\btheta_0$.
\end{definition}

\noindent Based on Theorem~\ref{thm:LLN}, sample average of the observation set is a consistent estimator of the true mean of the samples. Another important characteristic of estimators is based on the following bound over their covariance matrices:
\begin{theorem}[\bfseries Cram\'{e}r-Rao, \citet{lehmann1998theory}]
\label{thm:CR_bound}
Let $\bx_1,...,\bx_n\sim p(\bx|\btheta_0)$ and $\btheta_n=\btheta(\bx_1,...,\bx_n)$ be an estimator. If the first moment of $\btheta_n$ is differentiable with respect to the parameter vector and its second moment is finite, then the following inequality holds for every $\btheta\in\Omega$:
\begin{equation}
\label{eq:CR_bound}
\mbox{\normalfont{Cov}}[\btheta_n]\enskip \succeq \enskip -\left(\nabla_{\btheta}\mathbb{E}[\btheta_n] \right)^\top \bI(\btheta)^{-1} \nabla_{\btheta}\mathbb{E}[\btheta_n].
\end{equation}
\end{theorem}

\noindent The right-hand-side of~(\ref{eq:CR_bound}) is called the \emph{Cramer-Rao bound} of the estimator, where the middle term is the inverse of the \emph{Fisher information matrix} of the parametric distribution $p(\bx|\btheta)$: 
\begin{equation*}
\bI(\btheta) \equals \mathbb{E}\left[\nabla_{\btheta}\log p(\bx|\btheta)\cdot\nabla_{\btheta}^\top\log p(\bx|\btheta)\right]
\end{equation*}
Theorem~\ref{thm:CR_bound} suggests that for an unbiased estimator $\btheta_n$, the inequality over the covariance matrix becomes: $\mbox{Cov}[\btheta_n]\succeq\bI(\btheta)^{-1}, \forall \btheta\in\Omega$. 

\begin{definition}[\bfseries Efficiency]
We say that an estimator $\btheta_n$ is \emph{efficient}, iff it attains the Cramer-Rao bound, that is $\mbox{\normalfont Cov}[\btheta_n]$ achieves the lower-bound in~(\ref{eq:CR_bound}) for every $n=1,2,...$~. Furthermore, we say that $\btheta_n$ is \emph{asymptotically efficient}, iff the lower bound is attained asymptotically (when $n\to\infty$).
\end{definition}

\section{Proof of Second-order Multivariate Delta Method}
\label{app:proof_2nd_Delta_Method}
In order to prove this theorem, we have to formulate the statistical Taylor expansion. This, in turn, needs a brief introduction of stochastic order notations.
\subsection{Stochastic Order Notations}
\label{subapp:stochastic_order_notations}
The stochastic order notations are denoted by $o_p$ and $O_p$, where the former is equivalent to convergence in probability (Definition~\ref{defn:convergence_in_prob}) and the latter implies \emph{boundedness in probability}. In what follows, if otherwise stated, $\{\bz_n\}$ is a sequence of multivariate random variables lying in $\Omega\subseteq\mathbb{R}^d$ and $\{a_n\}$ is a sequence of strictly positive real numbers. The skipped proofs can be found in many textbooks on asymptotic theory, such as \citet[Chapter 6]{brockwell1991time}.
\begin{definition}
\label{eq:stochastic_small_o}
We write $\bz_n=o_p(a_n)$ iff
\begin{equation}
\frac{\theta_{in}}{a_n} \equals o_p(1), \quad \mbox{for all } i=1,...,d
\end{equation}
\end{definition}
\begin{definition}
We write $\bz_n=O_p(a_n)$ iff the sequence $\left\{\frac{\theta_{in}}{a_n} \right\}$ is bounded in probability for every $i=1,...,d$, that is, for every $\epsilon>0$ there exists $\delta_\epsilon$ such that
\begin{equation}
P\left(\left|\frac{\theta_{in}}{a_n}\right| > \delta_\epsilon \right) \enskip < \enskip \epsilon, \quad n=1,2,...
\end{equation}
\end{definition}
%
%
We also need the following propositions:
\begin{proposition}[\citet{brockwell1991time}]
\label{prop:convergence_types_relationship}
Let $\{\theta_n\}$ and $\{\eta_n\}$ be two sequences of scalar random variables, and $\{a_n\}$ and $\{b_n\}$ be two sequences of positive real numbers. If $\theta_n=O_p(a_n)$ and $\eta_n=o_p(b_n)$, then
\begin{enumerate}[label=(\roman*)]
\item $\theta_n^2 = O_p(a_n^2)$
\item $\theta_n\eta_n=o_p(a_nb_n)$
\end{enumerate}
\end{proposition}

\begin{proposition}
\label{prop:vector_norm_convergence}
The followings are true\footnote{Unless subscripted, $\|\cdot\|$ denotes the $L_2$ norm in all the equations.}: 
\begin{enumerate}[label=(\roman*)]
\item\label{item:norm_convergence_op}  $\bz_n=o_p(a_n) \enskip \Leftrightarrow \enskip \|\bz_n\|=o_p(a_n)$.
\item\label{item:norm_convergence_Op}  $\bz_n=O_p(a_n) \enskip \Leftrightarrow \enskip \|\bz_n\|=O_p(a_n)$.
\end{enumerate}
\end{proposition}
\begin{proof}
The proof of part~\ref{item:norm_convergence_op} can be found in~\citet[Proposition 6.1.2]{brockwell1991time}. Here, we only prove part~\ref{item:norm_convergence_Op}:

\vspace{.25cm}
\noindent $(ii, \Rightarrow)$ : Since $\bz_n=O_p(a_n)$, for every $\varepsilon>0$ and for every $i=1,...,d$, there exists a coefficient $\delta_i>0$ such that
\begin{equation}
P\left(|\theta_{ni}|>a_n\cdot\delta_i\right) \enskip < \enskip \frac{\varepsilon}{d} \quad,n=1,2,....
\end{equation}
Define $\delta_{\mbox{\footnotesize max}}=\max\{\delta_1,...,\delta_d\}$ and note that we can write
\begin{align}
\left\{\bz_n:\enskip \sum_{i=1}^d |\theta_{ni}|^2 > \left(d\cdot a_n\cdot\delta_{\mbox{\footnotesize max}}\right)^2 \right\} &\enskip \subseteq \enskip \left[\bigcap_{i=1}^d \left\{\bz_n:\enskip |\theta_{ni}|\leq a_n\cdot\delta_{\mbox{\footnotesize max}}\right\}\right]^c \nonumber\\
&\equals \bigcup_{i=1}^d \left\{\bz_n:\enskip |\theta_{ni}|>a_n\cdot\delta_{\mbox{\footnotesize max}}\right\}
\end{align}
implying that
\begin{align}
P\left(\|\bz_n\|^2 > \left(d\cdot a_n\cdot\delta_{\mbox{\footnotesize max}}\right)^2 \right) \enskip &\leq \enskip P\left( \bigcup_{i=1}^d \left\{\bz_n:\enskip |\theta_{ni}|>a_n\cdot\delta_{\mbox{\footnotesize max}}\right\} \right) \nonumber\\
&\leq \enskip \sum_{i=1}^dP(|\theta_{ni}|>a_n\cdot\delta_{\mbox{\footnotesize max}})
\end{align}
Furthermore, for every $i=1,...,d$ we have $\delta_{\mbox{\footnotesize max}}\geq\delta_i$, consequently the interval $(a_n\delta_{\mbox{\footnotesize max}},\infty)$ is a subset of $(a_n\delta_i,\infty)$ and $P(|\theta_{ni}|>a_n\delta_{\mbox{\footnotesize max}}) \leq P(|\theta_{ni}|>a_n\delta_i)$. This implies that 
\begin{equation}
P\left(\|\bz_n\|^2 > \left(d\cdot a_n\cdot\delta_{\mbox{\footnotesize max}}\right)^2 \right) \enskip \leq \enskip \sum_{i=1}^d P(|\theta_{ni}|>a_n\cdot\delta_i) \enskip < \enskip \varepsilon.
\end{equation}
Therefore, for every $\varepsilon>0$, we can choose $\delta_\varepsilon = d\cdot\delta_{\mbox{\footnotesize max}}$ such that $P\left(\frac{\|\bz_n\|}{a_n}>\delta_\varepsilon\right)<\varepsilon$ for every $n=1,2,...$, that is $\|\bz_n\|=O_p(a_n)$.

\noindent $(ii, \Leftarrow)$ : Suppose $\|\bz_n\|=O_p(a_n)$, that is for every $\varepsilon>0$ we can find $\delta_\varepsilon>0$ such that
\begin{equation}
P(\|\bz_n\|>a_n\cdot\delta_\varepsilon) \enskip < \varepsilon \quad, n=1,2,...
\end{equation}
It is clear that for any given $i\in\{1,...,d\}$ we have
\begin{equation}
\label{eq:component_norm_sets}
\left\{\bz_n:  |\theta_{ni}|>a_n\cdot\delta_\varepsilon \right\} \enskip \subseteq \enskip \left\{\bz_n: \|\bz_n\|>a_n\cdot\delta_\varepsilon \right\}
\end{equation}
hence
\begin{equation}
P(|\theta_{ni}|>a_n\cdot\delta_\varepsilon)\enskip \leq \enskip P(\|\bz_n\|>a_n\cdot\delta_\varepsilon) \enskip < \epsilon \quad, n=1,2,...
\end{equation}
meaning that $\theta_{ni} = O_p(a_n), i=1,...,d$ or equivalently $\bz_n = O_p(a_n)$.
\end{proof}

\begin{proposition}
\label{prop:Op_convergent_op}
If $\bz_n=O_p(a_n)$ and $a_n\to0 (n\to\infty)$, then $\bz_n=o_p(1)$.
\end{proposition}
\begin{proof}
The goal is to show $\bz_n=o_p(1)$ or equivalently $\|\bz_n\|=o_p(1)$ by proving that $P(\|\bz_n\|>\varepsilon)\to0 (n\to\infty)$ for every $\varepsilon>0$. Fix $\varepsilon$ to a positive real number. In order to have the sequence of probability numbers $\{P(\|\bz_n\|>\varepsilon)\}$ converging to zero, for every $\varepsilon_0>0$ there should exist a positive integer $N>0$ such that
\begin{equation}
\label{eq:convergent_to_zero}
P(\|\bz_n\|>\varepsilon)<\varepsilon_0 \quad \forall n>N.
\end{equation}
Because of the assumption of being bounded by $a_n$, that is $\bz_n=O_p(a_n)$ or equivalently $\|\bz_n\|=O_p(a_n)$, we can choose a real number $\delta_0>0$ such that
\begin{equation}
\label{eq:bounded_by_an}
P(\|\bz_n\|>a_n\delta_0)\enskip < \enskip \varepsilon_0 \quad n=1,2,...
\end{equation}
On the other hand, since $a_n\to0(n\to\infty)$, there exists a large enough number $N_0>0$ such that $0<a_n<\frac{\varepsilon}{\delta_0}$ for all $n>N_0$. Therefore we get:
\begin{equation}
[0,a_n\delta_0] \enskip \subseteq \enskip [0,\varepsilon] \quad \forall n>N_0
\end{equation}
implying that
\begin{equation}
\label{eq:N0_inequality}
P(\|\bz_n\|\leq a_n\delta_0) \enskip \leq \enskip P(\|\bz_n\|\leq\varepsilon) \quad \forall n>N_0
\end{equation}
From inequalities~(\ref{eq:bounded_by_an}) and~(\ref{eq:N0_inequality}), and noticing that the latter holds for all $n$ whereas the former is satisfied when $n>N_0$, one can write:
\begin{equation}
P(\|\bz_n\|>\varepsilon) \enskip \leq \enskip P(\|\bz_n\|>a_n\delta_0)\enskip<\enskip\varepsilon_0 \quad \forall n>N_0
\end{equation}
Therefore, for every $\varepsilon_0>0$, equation~(\ref{eq:convergent_to_zero}) is guaranteed if $N$ is chosen to be equal to $N_0$ so that inequality~(\ref{eq:N0_inequality}) is satisfied. Similarly, this can be written for every $\varepsilon>0$, thus the proof is complete.
\end{proof}

\begin{proposition}[\citet{serfling2009approximation}, Chapter 1]
\label{prop:boundedness_convergence_in_law}
Let $\{\bz_n\}$ be a sequence of random variables. If there exists a random variable $\bz_0$ such that $\bz_n\overset{L}{\to}\bz_0$, then $\bz_n=O_p(1)$.
\end{proposition}

\subsection{Second-order Statistical Taylor Expansion}
\label{subapp:2nd_statisticaly_Taylor}
Now we are ready to establish the second-order  statistical Taylor expansion. 
\begin{theorem}
\label{thm:taylor_in_probability_2nd}
Let $\{\bz_n\}$ be a sequence of random vectors in a convex and compact set $\Omega\subseteq \mathbb{R}^d$ and $\bz_0\in \Omega$ be a constant vector such that $\bz_n - \bz_0 = O_p(a_n)$ where $a_n\to0 (n\to\infty)$. If $g:\Omega\to\mathbb{R}$ is a $\calc^3$ function , then 
\begin{equation}
g(\bz_n) \equals g(\bz_0) \enplus \nabla_{\bz}^\top g(\bz_0)(\bz_n-\bz_0) \enplus \frac{1}{2}(\bz_n-\bz_0)^\top\nabla_{\bz}^2g(\bz_0)(\bz_n-\bz_0) \enplus o_p(a_n^2).
\end{equation}
\end{theorem}
\begin{proof}
Since $g$ is twice continuously differentiable in a neighborhood of $\bz_0$, it can be written in terms of the Taylor expansion as
\begin{equation}
g(\bz) \equals g(\bz_0) \enplus (\bz-\bz_0)^\top\nabla_{\bz}g(\bz_0) \enplus \frac{1}{2}(\bz-\bz_0)^\top\nabla_{\bz}^2g(\bz_0)(\bz-\bz_0) \enplus r_2(\bz,\bz_0)
\end{equation}
where $r_2(\bz,\bz_0)$ is the Lagrange remainder of second order. Based on Taylor's polynomial theorem for multivariate functions, there exists a number $t\in[0,1]$ such that $\bz^*=t\bz + (1-t)\bz_0\in\Omega$ (due to convexity of $\Omega$) and 
\begin{equation}
r_2(\bz,\bz_0) \equals \frac{1}{6}\sum_{1\leq i,j,k\leq d} \frac{\partial^3 g(\bz^*)}{\partial \theta_i\partial \theta_j\partial \theta_k}(\theta_i-\theta_{0i})(\theta_j-\theta_{0j})(\theta_k-\theta_{0k})
\end{equation}
But since $\Omega$ is compact and $g\in\calc^3$, the third derivative of $g$ is bounded\footnote{This is because of the following Theorem in real analysis:
\begin{theorem}
Let $X$ and $Y$ be two vector spaces. If $g:X\to Y$ is continuous and $X$ is compact, then $f(X)$ is compact in $Y$.
\end{theorem}
In special case of this theorem, when $Y=\mathbb{R}$, compactness of $f(X)$ is equivalent to boundedness and closedness.} and therefore there exists $M>0$ such that
\begin{equation}
\left|\frac{\partial^3 g(\bz)}{\partial \theta_i\partial \theta_j\partial \theta_k}\right| \enskip \leq \enskip M \quad,\forall \bz\in \Omega\enskip,\forall i,j,k\in\{1,...,d\}
\end{equation}
Hence the Lagrange remainder can be bounded by
\begin{align}
\label{eq:bounded_lagrange_remainder}
|r_2(\bz,\bz_0)| &\enskip \leq \enskip \frac{M}{6}\sum_{1\leq i,j,k\leq3} |\theta_i-\theta_{0i}|\cdot|\theta_j-\theta_{0j}|\cdot|\theta_k-\theta_{0k}| \nonumber\\
&\equals \frac{M}{6}\|\bz-\bz_0\|_1^3\nonumber\\
&\enskip\leq\enskip \frac{M'}{6}\|\bz-\bz_0\|^3
\end{align}
where $M'=c_uM$ and $c_u$ is obtained from the equivalence of norms in $\mathbb{R}^d$ \footnote{Two norm functions $\|\cdot\|_{(1)}$ and $\|\cdot\|_{(2)}$, in a vector space $\Omega$, are called \emph{equivalent} iff there exist constants $c_u\geq c_d>0$ such that
\begin{equation}
c_d\|\bz\|_{(2)} \enskip \leq \enskip \|\bz\|_{(1)} \enskip \leq \enskip c_u\|\bz\|_{(2)} \quad,\forall \bz\in \Omega.
\end{equation}}. Now define the function $h:\Omega\to\mathbb{R}$ as below
\begin{equation}
\label{eq:definition_hz}
h(\bz) \enskip := \enskip \left\{\begin{array}{ll}
\displaystyle \frac{r_2(\bz,\bz_0)}{\|\bz-\bz_0\|^2/2} & ,\bz\neq\bz_0 \\[.25cm]
0 & ,\bz=\bz_0
\end{array}\right.
\end{equation}
Note that $h(\bz)$ is continuous at $\bz=\bz_0$: due to boundedness of $r_2(\bz,\bz_0)$, $h(\bz)$ is also bounded by
\begin{equation}
|h(\bz)| \enskip \leq\enskip \frac{M'}{3}\|\bz-\bz_0\|.
\end{equation}
Hence, for every $\varepsilon>0$, we can select $\delta_\varepsilon=\frac{3\varepsilon}{M'}$ such that the following continuity condition holds
\begin{equation}
\|\bz-\bz_0\|<\delta_\varepsilon \enskip \Rightarrow\enskip |h(\bz)| \enskip \leq \enskip \varepsilon.
\end{equation}
Continuity of $h(\bz)$ at $\bz=\bz_0$ implies $\lim_{\bz\to\bz_0}h(\bz)=h(\bz_0)=0$. Furthermore, since $\bz_n-\bz_0=O_p(a_n)$ and $a_n\to0(n\to\infty)$, Proposition~\ref{prop:Op_convergent_op} suggests that $\bz_n-\bz_0=o_p(1)$. These two enable us to use Proposition~\ref{prop:convergent_continuous_mappings} and write
\begin{equation}
h(\bz_n) - h(\bz_0) \equals h(\bz_n) \equals o_p(1).
\end{equation}
Finally, from equation~(\ref{eq:definition_hz}) and Propositions~\ref{prop:convergence_types_relationship},~\ref{prop:vector_norm_convergence} and~\ref{prop:Op_convergent_op}, we can write that
\begin{equation}
r_2(\bz_n,\bz_0) \equals h(\bz_n)\cdot\frac{\|\bz_n-\bz_0\|^2}{2}\equals o_p(1)\cdot O_p(a_n^2)\equals o_p(a_n^2)
\end{equation}
\end{proof}

\subsection{Second-order Multivariate Delta Method}
\label{subapp:2nd_multivariate_delta_method}
Finally, here is the proof of second-order multivariate Delta method (Theorem~\ref{thm:delta_method_zeroder}):

\begin{proof}
From assumption of the Theorem, $\sqrt{n}(\bz_n-\bz_0)\overset{L}{\to}\caln(\mathbf{0},\bSigma)$, and Proposition~\ref{prop:boundedness_convergence_in_law}, one conclude that $\sqrt{n}(\bz_n-\bz_0)=O_p(1)$ and therefore  $\bz_n-\bz_0 = O_p\left(\frac{1}{\sqrt{n}}\right)$. Thus we can use Theorem~\ref{thm:taylor_in_probability_2nd} with $a_n=\frac{1}{\sqrt{n}}$ to write:
\begin{equation}
g(\bz) \equals g(\bz_0) \enplus (\bz-\bz_0)^\top\nabla_{\bz}g(\bz_0) \enplus \frac{1}{2}(\bz-\bz_0)^\top\nabla_{\bz}^2g(\bz_0)(\bz-\bz_0) \enplus o_p\left(\frac{1}{n}\right),
\end{equation}
hence
\begin{align}
n\bigg[g(\bz) - g(\bz_0)\bigg] &\equals \frac{1}{2}\left[\sqrt{n}\cdot(\bz-\bz_0)\right]^\top\nabla_{\bz}^2g(\bz_0)\left[\sqrt{n}\cdot(\bz-\bz_0)\right] \enplus o_p(1) \nonumber\\
&\enskip \overset{L}{\to} \enskip \frac{1}{2}\caln(\mathbf{0},\bSigma)^\top \nabla_{\bz}^2g(\bz_0) \caln(\mathbf{0},\bSigma)\nonumber\\
&\equals \frac{1}{2}\caln(\mathbf{0},\mathbb{I}_d)^\top \left[\bSigma^{1/2}\nabla_{\bz}^2g(\bz_0)\bSigma^{1/2} \right]\caln(\mathbf{0},\mathbb{I}_d)
\end{align}
Define $\boldsymbol{\Gamma} := \bSigma^{1/2}\nabla_{\bz}^2g(\bz_0)\bSigma^{1/2}$ and rewrite the right-hand-side element-wise as
\begin{align}
\frac{1}{2}\caln(\mathbf{0},\mathbb{I}_d)^\top \boldsymbol{\Gamma} \caln(\mathbf{0},\mathbb{I}_d) \equals \frac{1}{2}\sum_{i=1}^d \lambda_{i} \caln(0,1)^2 \equals \frac{1}{2}\sum_{i=1}^d\lambda_{i}\chi_1^2,
\end{align}
where $\lambda_i$'s are eigenvalues of $\boldsymbol{\Gamma}$. Finally, noting that the terms in the Chi-square mixture are independent, variance of the convergent random variable can be easily computed as
\begin{align}
\label{eq:asymptotic_variance_zeroder}
\mbox{Var}\left[\frac{1}{2}\sum_{i=1}^d  \lambda_{i} \chi_1^2  \right] &\equals \frac{1}{4}\sum_{i=1}^d\lambda_{i}^2\cdot\mbox{Var}\left[\chi_1^2\right]\nonumber\\
&\equals \frac{1}{2}\sum_{i=1}^d\lambda_{i}^2 \nonumber\\
&\equals \frac{1}{2}\left\|\bSigma^{1/2}\nabla_{\bx}^2g(\bx_0)\bSigma^{1/2} \right\|_F^2,
\end{align}
\end{proof}

\section{Proof of Lemma~\ref{lemma:equivalent_set_optimization}}
\label{app:proof_Lemma}
We first substitute the score function of the classifier
\begin{equation}
\nabla_{\btheta} \log p(y|\bx,\btheta) \equals \frac{\nabla_{\btheta} p(y|\bx,\btheta)}{p(y|\bx,\btheta)}\nonumber\\
\end{equation}
into formulation Monte-Carlo approximation of $\bI_q$ to get:
\begin{align}
\hat{\bI}(\btheta;X_q) &\equals \frac{1}{|X_q|}\sum_{\bx\in X_q}\sum_{y=1}^cp(y|\bx,\btheta)\cdot\frac{\nabla_{\btheta} p(y|\bx,\btheta)\nabla_{\btheta}^\top p(y|\bx,\btheta)}{p(y|\bx,\btheta)^2}  \enplus \delta\mathbb{I}_d \\ 
&\equals \frac{1}{|X_q|}\sum_{\bx\in X_q} \sum_{y=1}^c\frac{\nabla_{\btheta}p(y|\bx,\btheta).\nabla_{\btheta}^\top p(y|\bx,\btheta)}{p(y|\bx,\btheta)} \enplus \delta\mathbb{I}_d
\end{align}
Define the vector $\bv_{\btheta}(\bx,y):=\nabla_{\btheta}p(y|\bx,\btheta)\bigg/\sqrt{p(y|\bx,\btheta)}$ and rewrite $\hat{\bI}(\btheta;X_q)$ as:
\begin{equation}
\label{eq:Iq_v}
\hat{\bI}(\btheta;X_q) \equals \frac{1}{|X_q|}\sum_{\bx\in X_q} \sum_{y=1}^c \bv_{\btheta}(\bx,y).\bv_{\btheta}(\bx,y)^\top +\delta\cdot\mathbb{I}_d.
\end{equation}
On the other hand, since $X_q\subset X_p$ we can write $\hat{\bI}(\btheta;X_p)$ in terms of $\hat{\bI}(\btheta;X_q)$ by breaking the summation over $X_p$ into summations over $X_q$ and $X_p-X_q$ as follows:
\begin{align}
\hat{\bI}(\btheta;X_p) &\equals \frac{|X_q|}{|X_p|}\left[\frac{1}{|X_q|}\sum_{\bx\in X_q} \sum_{y=1}^c \bv_{\btheta}(\bx,y).\bv_{\btheta}(\bx,y)^\top +\delta\cdot\mathbb{I}_d \right] \nonumber\\
&\phantom{\equals} \enplus \frac{1}{|X_p|}\sum_{\bx\in X_p-X_q} \sum_{y=1}^c \bv_{\btheta}(\bx,y).\bv_{\btheta}(\bx,y)^\top \enplus \delta\left(\frac{|X_p|-|X_q|}{|X_p|}\right)\cdot \mathbb{I}_d \nonumber\\
&\equals \left(\frac{|X_q|}{|X_p|}\right)\cdot\hat{\bI}(\btheta;X_q) \enplus \frac{1}{|X_p|}\sum_{\bx\in X_p-X_q} \sum_{y=1}^c \bv_{\btheta}(\bx,y).\bv_{\btheta}(\bx,y)^\top\nonumber\\
&\phantom{\equals} \enplus \delta\left(\frac{|X_p|-|X_q|}{|X_p|}\right)\cdot\mathbb{I}_d
\end{align}
Now that we related the Fisher information matrices to each other, we can compute the product of $\hat{\bI}(\btheta;X_p)$ and $\hat{\bI}(\btheta;X_q)^{-1}$:
\begin{align}
\hat{\bI}(\btheta;X_q)^{-1}\hat{\bI}(\btheta;X_p) &\equals \left(\frac{|X_q|}{|X_p|}\right)\cdot\mathbb{I}_d \enplus \frac{\hat{\bI}(\btheta;X_q)^{-1}}{|X_p|}\left[\sum_{\bx\in X_p-X_q} \sum_{y=1}^c \bv_{\btheta}(\bx,y)\cdot\bv_{\btheta}(\bx,y)^\top \right] \nonumber\\
&\phantom{\equals} \enplus \delta\left(\frac{|X_p|-|X_q|}{|X_p|}\right)\cdot\hat{\bI}(\btheta;X_q)^{-1}
\end{align}
Applying the trace function to both sides of the equation will result:
\begin{align}
\label{eq:objective_wtrace}
\mbox{tr}\left[\hat{\bI}(\btheta;X_q)^{-1}\hat{\bI}(\btheta;X_p)\right] &\equals \frac{|X_q|\cdot d}{|X_p|} + \frac{1}{|X_p|}\sum_{\bx\in X_p-X_q} \sum_{y=1}^c \mbox{tr}\left[\hat{\bI}(\btheta;X_q)^{-1} \bv_{\btheta}(\bx,y)\cdot\bv_{\btheta}(\bx,y)^\top\right]\nonumber\\
&\phantom{\equals} \enplus \delta\left(\frac{|X_p|-|X_q|}{|X_p|}\right)\cdot\mbox{tr}\left[\hat{\bI}(\btheta;X_q)^{-1}\right] \nonumber\\
&\enskip \approx\enskip  \frac{|X_q|\cdot d}{|X_p|} \enplus \frac{1}{|X_p|}\sum_{\bx\in X_p-X_q} \sum_{y=1}^c \bv_{\btheta}(\bx,y)^\top\hat{\bI}(\btheta;X_q)^{-1} \bv_{\btheta}(\bx,y),
\end{align}
where the last term is dropped since the overloading constant, $\delta$, is assumed to be small. Furthermore, the term including $\hat{\bI}(\btheta;X_q)^{-1}$ can be approximated by replacing the weighted harmonic mean of the eigenvalues of $\hat{\bI}(\btheta;X_q)$ by their weighted arithmetic mean~\citep{Hoi2006}:
\begin{equation}
\label{eq:inv_Iq_approx}
\bv_{\btheta}(\bx,y)^\top\hat{\bI}(\btheta;X_q)^{-1} \bv_{\btheta}(\bx,y) \enskip \approx \enskip \frac{\|\bv_{\btheta}(\bx,y)\|^4}{\bv_{\btheta}(\bx,y)^\top\hat{\bI}(\btheta;X_q)\bv_{\btheta}(\bx,y)}.
\end{equation}
Note that this approximation becomes exact when the condition number of $\hat{\bI}(\btheta;X_q)$ is one. Substituting $\hat{\bI}(\btheta;X_q)$ from equation~(\ref{eq:Iq_v}) into the denominator of the approximation above yields:
\begin{equation}
\bv_{\btheta}(\bx,y)^\top\hat{\bI}(\btheta;X_q)\bv_{\btheta}(\bx,y) \equals \frac{1}{|X_q|}\sum_{\bx'\in X_q}\sum_{y'=1}^c\left[\bv_{\btheta}(\bx,y)^\top \bv_{\btheta}(\bx',y') \right]^2 \enplus \delta\|\bv_{\btheta}(\bx,y)\|^2
\end{equation}
Integrating this approximation with equation~(\ref{eq:objective_wtrace}), and assuming that the value of $\btheta$ is not located at the stationary point of the conditional density $p(y|\bx,\btheta)$ (hence $\bv_{\btheta}(\bx,y)$ is not the zero vector), results:
\begin{align}
\label{eq:objective_wtrace_simp}
\mbox{tr}&\left[\hat{\bI}(\btheta;X_q)^{-1}\hat{\bI}(\btheta;X_p)\right] \enskip\approx\enskip \frac{|X_q|\cdot d}{|X_p|} \nonumber\\
&+\enskip \frac{1}{|X_p|}\sum_{\bx\in X_p-X_q}\sum_{y=1}^c\frac{1}{\delta\cdot\|\bv_{\btheta}(\bx,y)\|^{-2} + \sum_{\bx'\in X_q} g_{\btheta}(\bx,y,\bx')}
\end{align}
where
\begin{equation}
\label{eq:defn_g}
g_{\btheta}(\bx,y,\bx') \enskip := \enskip \frac{1}{|X_q|}\sum_{y'=1}^c \left[\frac{\bv_{\btheta}(\bx,y)^\top\bv_{\btheta}(\bx',y')}{\|\bv_{\btheta}(\bx,y)\|^2}\right]^2
\end{equation}
Finally in~(\ref{eq:objective_wtrace_simp}), removing the constants we get
\begin{align}
\label{eq:set_maximization}
\argmin_{X_q\subset X_p} \enskip &\mbox{tr}\left[\hat{\bI}(\btheta;X_q)^{-1}\hat{\bI}(\btheta;X_p)\right] \nonumber\\
&\approx \enskip \operatorname*{arg\,max}_{X_q\subset X_p} \sum_{\bx\in X_p-X_q}\sum_{y=1}^c\frac{-1}{\delta\cdot\|\bv_{\btheta}(\bx,y)\|^{-2} + \sum_{\bx'\in X_q} g_{\btheta}(\bx,y,\bx')}
\end{align}

\section{Proof of Theorem~\ref{thm:submodularity}}
\label{app:proof_submodularity}
Proof of this Theorem is a generalization of the discussion by~\citet{Hoi2006}, with clarification of all the assumptions and approximations made.

First, note that the function $f_{\btheta}$ can be broken into simpler terms $f_{\btheta}(X_q)=\sum_{y=1}^cf_{\btheta}(X_q;y)$, where
\begin{equation}
f_{\btheta}(X_q;y) \equals \sum_{\bx\in X_p-X_q}\frac{-1}{\delta\cdot\|v_{\btheta}(\bx,y)\|^{-2} + \sum_{\bx'\in X_q} g_{\btheta}(\bx,y,\bx')},\quad \forall X_q\subseteq X_p.
\end{equation}
Therefore, in order to prove submodularity and monotonicity of $f_{\btheta}$, it suffices to prove these properties for $f_{\btheta}(\cdot;y)$ for all $y\in\{1,...,c\}$. 
Fix $y$ and take any subset $X_q\subseteq X_p$ and $\bxi\in X_p-X_q$. Then, we can write:
\begin{eqnarray}
f_{\btheta}(X_q\cup\{\bxi\};y) &\equals &\sum_{\bx\in X_p-(X_q\cup\{\bxi\})}\frac{-1}{\delta\cdot\|v_{\btheta}(\bx,y)\|^{-2} + \sum_{\bx'\in X_q\cup\{\bxi\}} g_{\btheta}(\bx,y,\bx')} \\
&\equals &\sum_{\bx\in X_p-X_q}\frac{-1}{\delta\cdot\|v_{\btheta}(\bx,y)\|^{-2} + \sum_{\bx'\in X_q\cup\{\bxi\}} g_{\btheta}(\bx,y,\bx')} \nonumber\\
& &+\enskip \frac{1}{\delta\cdot\|v_{\btheta}(\bxi,y)\|^{-2} + \sum_{\bx'\in X_q\cup\{\bxi\}} g_{\btheta}(\bxi,y,\bx')}. \nonumber
\end{eqnarray}
We then form the discrete derivative of $f_{\btheta}(\cdot;y)$ at $X_q$ to get:
\begin{align}
\label{eq:discrete_derivative_fy}
&\rho_{f_{\btheta}(\cdot;y)}(X_q;\bxi) \equals f_{\btheta}(X_q\cup\{\bxi\};y) \enskip-\enskip f_{\btheta}(X_q;y) \nonumber\\
&\equals \sum_{\bx\in X_p-X_q}\left[ \frac{-1}{\frac{\delta}{\|\bv_{\btheta}(\bx,y)\|^{2}} + \sum_{\bx'\in X_q\cup\{\bxi\}} g_{\btheta}(\bx,y,\bx')} + \frac{1}{\frac{\delta}{\|\bv_{\btheta}(\bx,y)\|^{2}} + \sum_{\bx'\in X_q} g_{\btheta}(\bx,y,\bx')} \right] \nonumber\\
&\phantom{\equals} \enplus \frac{1}{\frac{\delta}{\|\bv_{\btheta}(\bx,y)\|^{2}} + \sum_{\bx'\in X_q\cup\{\bxi\}} g_{\btheta}(\bxi,y,\bx')}.
\end{align}
The right-hand-side can be rewritten as
\begin{align}
\label{eq:discrete_derivative_fy_re}
&\sum_{\bx\in X_p-X_q}\left[ \frac{g_{\btheta}(\bx,y,\bxi)}{\left(\frac{\delta}{\|\bv_{\btheta}(\bx,y)\|^{2}}+\sum_{\bx'\in X_q\cup\{\bxi\}} g_{\btheta}(\bx,y,\bx')\right)\left(\frac{\delta}{\|\bv_{\btheta}(\bx,y)\|^{2}}+\sum_{\bx'\in X_q} g_{\btheta}(\bx,y,\bx')\right)} \right]\nonumber\\
&\phantom{\equals} \enplus \frac{1}{\frac{\delta}{\|\bv_{\btheta}(\bx,y)\|^{2}}+\sum_{\bx'\in X_q\cup\{\bxi\}} g_{\btheta}(\bxi,y,\bx')}.
\end{align}
Since by definition $g_{\btheta}(\bx,y,\bx')\geq0,\forall\bx,y,\bx'$, all of the terms in~(\ref{eq:discrete_derivative_fy_re}) are non-negative and therefore $\rho_{f_{\btheta}(\cdot;y)}(X_q;\bxi)\geq0$. This is true for any $X_q\subseteq X_p$ hence monotonicity of $f_{\btheta}(\cdot;y)$ is obtained.
Now let us take any superset $X_{q'}$ such that $X_{q}\subseteq X_{q'}\subseteq X_p$ and $\bxi\in X_p-X_{q'}$, and form the difference between their corresponding discrete derivatives. From~(\ref{eq:discrete_derivative_fy_re}) we will have:
\begin{align}
\label{eq:discrete_derivative_diff}
& \rho_{f_{\btheta}(\cdot;y)}(X_q;\bxi) - \rho_{f_{\btheta}(\cdot;y)}(X_{q'};\bxi) \nonumber\\
&= \sum_{\bx\in X_p-X_q}\left[ \frac{g_{\btheta}(\bx,y,\bxi)}{\left(\frac{\delta}{\|\bv_{\btheta}(\bx,y)\|^{2}}+\sum_{\bx'\in X_q\cup\{\bxi\}} g_{\btheta}(\bx,y,\bx')\right)\left(\frac{\delta}{\|\bv_{\btheta}(\bx,y)\|^{2}}+\sum_{\bx'\in X_q} g_{\btheta}(\bx,y,\bx')\right)} \right] \nonumber\\
& + \frac{1}{\frac{\delta}{\|\bv_{\btheta}(\bxi,y)\|^{2}}+\sum_{\bx'\in X_q\cup\{\bxi\}} g_{\btheta}(\bxi,y,\bx')} \nonumber \\
&-  \sum_{\bx\in X_p-X_{q'}}\left[ \frac{g_{\btheta}(\bx,y,\bxi)}{\left(\frac{\delta}{\|\bv_{\btheta}(\bx,y)\|^{2}}+\sum_{\bx'\in X_{q'}\cup\{\bxi\}} g_{\btheta}(\bx,y,\bx')\right)\left(\frac{\delta}{\|\bv_{\btheta}(\bx,y)\|^{2}}+\sum_{\bx'\in X_{q'}} g_{\btheta}(\bx,y,\bx')\right)} \right] \nonumber\\
& -  \frac{1}{\frac{\delta}{\|\bv_{\btheta}(\bxi,y)\|^{2}}+\sum_{\bx'\in X_{q'}\cup\{\bxi\}} g_{\btheta}(\bxi,y,\bx')}.
\end{align}
From non-negativity of $g_{\btheta}$ and that $X_q\subseteq X_{q'}$, we can conclude that for any $\bx\in X$ and $y\in\{1,...,c\}$:
\begin{align}
\label{eq:inequality_q}
\sum_{\bx'\in X_{q'}} g_{\btheta}(\bx,y,\bx') \enskip &\geq \enskip \sum_{\bx'\in X_q} g_{\btheta}(\bx,y,\bx') \nonumber \\
\Leftrightarrow \quad \left[\sum_{\bx'\in X_{q'}} g_{\btheta}(\bx,y,\bx')+\frac{\delta}{\|\bv_{\btheta}(\bx,y)\|^{2}}\right]^{-1} \enskip &\leq \enskip \left[\sum_{\bx'\in X_q} g_{\btheta}(\bx,y,\bx')+\frac{\delta}{\|\bv_{\btheta}(\bx,y)\|^{2}}\right]^{-1} \nonumber\\
\Leftrightarrow \quad -\left[\sum_{\bx'\in X_{q'}} g_{\btheta}(\bx,y,\bx')+\frac{\delta}{\|\bv_{\btheta}(\bx,y)\|^{2}}\right]^{-1} \enskip &\geq \enskip -\left[\sum_{\bx'\in X_q} g_{\btheta}(\bx,y,\bx')+\frac{\delta}{\|\bv_{\btheta}(\bx,y)\|^{2}}\right]^{-1}
\end{align}
Similarly, since $X_q\cup\{\bxi\} \subseteq X_{q'}\cup\{\bxi\}$ we will get:
\begin{equation}
\label{eq:inequality_q_bxi}
\phantom{\Rightarrow \quad }-\left[\sum_{\bx'\in X_{q'}\cup\{\bxi\}} g_{\btheta}(\bx,y,\bx')+\frac{\delta}{\|\bv_{\btheta}(\bx,y)\|^{2}}\right]^{-1} \geq -\left[\sum_{\bx'\in X_q\cup\{\bxi\}} g_{\btheta}(\bx,y,\bx')+\frac{\delta}{\|\bv_{\btheta}(\bx,y)\|^{2}}\right]^{-1}
\end{equation}
Applying the inequalities~(\ref{eq:inequality_q}) and~(\ref{eq:inequality_q_bxi}) into euqation~(\ref{eq:discrete_derivative_diff}) results:
\begin{align}
\label{eq:ineq_delta_rho}
& \rho_{f_{\btheta}(\cdot;y)}(X_q;\bxi) - \rho_{f_{\btheta}(\cdot;y)}(X_{q'};\bxi)\nonumber \\
&\geq \sum_{\bx\in X_p-X_q}\left[ \frac{g_{\btheta}(\bx,y,\bxi)}{\left(\frac{\delta}{\|\bv_{\btheta}(\bx,y)\|^{2}}+\sum_{\bx'\in X_q\cup\{\bxi\}} g_{\btheta}(\bx,y,\bx')\right)\left(\frac{\delta}{\|\bv_{\btheta}(\bx,y)\|^{2}}+\sum_{\bx'\in X_q} g_{\btheta}(\bx,y,\bx')\right)} \right]  \nonumber \\
&-  \sum_{\bx\in X_p-X_{q'}}\left[ \frac{g_{\btheta}(\bx,y,\bxi)}{\left(\frac{\delta}{\|\bv_{\btheta}(\bx,y)\|^{2}}+\sum_{\bx'\in X_q\cup\{\bxi\}} g_{\btheta}(\bx,y,\bx')\right)\left(\frac{\delta}{\|\bv_{\btheta}(\bx,y)\|^{2}}+\sum_{\bx'\in X_q} g_{\btheta}(\bx,y,\bx')\right)} \right] \nonumber \\
& + \frac{1}{\frac{\delta}{\|\bv_{\btheta}(\bxi,y)\|^{2}}+\sum_{\bx'\in X_q\cup\{\bxi\}} g_{\btheta}(\bxi,y,\bx')} \enskip-\enskip \frac{1}{\frac{\delta}{\|\bv_{\btheta}(\bxi,y)\|^{2}}+\sum_{\bx'\in X_q\cup\{\bxi\}} g_{\btheta}(\bxi,y,\bx')}
\end{align}
which yields\footnote{The inequality in~(\ref{eq:ineq_delta_rho}) is obtained by the fact that, for every four positive real numbers $a$, $a_0$, $b$ and $b_0$, if we have $-a\geq-a_0$ and $-b\geq-b_0$ (similar to~(\ref{eq:inequality_q}) and~(\ref{eq:inequality_q_bxi})), then $$-a\cdot b\equals(-a)\cdot b\enskip\geq\enskip(-a_0)\cdot b\equals a_0\cdot(-b)\enskip\geq\enskip a_0\cdot(-b_0)\equals-a_0\cdot b_0.$$}
\begin{equation}
\label{eq:submodular_inequality}
\sum_{\bx\in X_{q'}-X_q} \frac{g_{\btheta}(\bx,y,\bxi)}{\left(\frac{\delta}{\|\bv_{\btheta}(\bx,y)\|^{2}}+\sum_{\bx'\in X_q\cup\{\bxi\}} g_{\btheta}(\bx,y,\bx')\right)\left(\frac{\delta}{\|\bv_{\btheta}(\bx,y)\|^{2}}+\sum_{\bx'\in X_q} g_{\btheta}(\bx,y,\bx')\right)}\enskip \geq \enskip 0.
\end{equation}
Inequality~(\ref{eq:submodular_inequality}) holds for any $X_q\subseteq X_p$; hence submodularity of $f_{\btheta}(\cdot;y)$ stands for all $y\in\{1,...,c\}$ and $\btheta\in\Omega$.